\documentclass[preprint,3p]{elsarticle}

\usepackage{amssymb}
\usepackage{amsmath}
\usepackage{amsthm}

\usepackage{amssymb}            
\usepackage{mathtools}          
\usepackage{mathrsfs}           
\usepackage{graphicx}           
\usepackage{subcaption}         
\usepackage[space]{grffile}     
\usepackage{url}                
\usepackage{lipsum}             
\usepackage{amsthm}
\usepackage{tikz}

\usepackage{cleveref}
\newtheorem{theorem}{Theorem}[section]
\newtheorem{proposition}[theorem]{Proposition}
\newtheorem{lemma}{Lemma}
\newtheorem{corollary}[theorem]{Corollary}
\newtheorem{definition}[theorem]{Definition}
\newtheorem{assumption}{Assumption}

\newtheorem{remark}[theorem]{Remark}
\usepackage{comment}
\crefname{assumption}{Assumption}{Assumptions}
\Crefname{assumption}{Assumption}{Assumptions}

\crefname{theorem}{Theorem}{Theorems}
\Crefname{theorem}{Theorem}{Theorems}
\crefname{lemma}{Lemma}{Lemmas}
\Crefname{lemma}{Lemma}{Lemmas}
\crefname{corollary}{Corollary}{Corollaries}
\Crefname{corollary}{Corollary}{Corollaries}

\journal{Neurocomputing}

\begin{document}

\begin{frontmatter}



\title{Finite-Time Analysis of Simultaneous Double Q-learning}



\author[inst1]{Hyunjun Na}
\ead{nhjun@kaist.ac.kr}

\author[inst1]{Donghwan Lee\corref{cor1}}
\ead{donghwan@kaist.ac.kr}

\cortext[cor1]{Corresponding author}

\address[inst1]{%
  Department of Electrical Engineering, Korea Advanced Institute of Science and Technology,\\
  291 Daehak‐ro, Daejeon 34141, South Korea%
}


\begin{abstract}
$Q$-learning is one of the most fundamental reinforcement learning (RL) algorithms. Despite its widespread success in various applications, it is prone to overestimation bias in the $Q$-learning update. To address this issue, double $Q$-learning employs two independent $Q$-estimators which are randomly selected and updated during the learning process. This paper proposes a modified double $Q$-learning, called simultaneous double $Q$-learning (SDQ), with its finite-time analysis. SDQ eliminates the need for random selection between the two $Q$-estimators, and this modification allows us to analyze double $Q$-learning through the lens of a novel switching system framework facilitating efficient finite-time analysis. Empirical studies demonstrate that SDQ converges faster than double $Q$-learning while retaining the ability to mitigate the maximization bias. Finally, we derive a finite-time expected error bound for SDQ. 
\end{abstract}




\end{frontmatter}



\section{Introduction}\label{sec:intro}

Reinforcement learning (RL) is a class of learning algorithms for finding an optimal policy in unknown environments through interactions with the environment~\citep{sutton2018reinforcement}.
Among them, $Q$-learning~\citep{watkins} is one of the most widely studied and practically successful methods, which aims to learn an optimal policy by iteratively estimating the optimal action–value function.
Owing to its simplicity and model-free nature, $Q$-learning has been successfully applied to a wide range of problems, including control, robotics, and game playing~\citep{park2001modular,zhou2005cooperative,ho2006hiq,mnih2013playing}.
From a theoretical perspective, its convergence properties have also been extensively studied under various setting, such as stochastic approximation frameworks and finite-time anlayses~\citep{borkar2000ode,even2003learning,azar2011speedy,zou2019finite,lee2020unified,lee2023discrete,lee2024final}.

Despite its empirical successes and theoretical achievements, $Q$-learning is known to suffer from overestimation in the $Q$-estimator, known as the maximization bias~\citep{sutton2018reinforcement}.
This bias arises because the $Q$-value update selects the maximum action-value estimate, often leading to overestimation due to noise in the sampled estimates. For instance, when multiple actions are available, even small overestimations can accumulate through repeated updates, systematically skewing the $Q$-function. This issue becomes particularly severe in environments with a large number of actions or heterogeneous action spaces, where it can significantly slow the convergence of the policy to an optimal solution. To overcome this obstacle, the so-called double $Q$-learning was proposed in~\cite{hasselt2010double}, which empirically demonstrated that the maximization bias can be reduced by using double $Q$-estimators instead of the single $Q$-estimator. Since its introduction, double $Q$-learning has been successfully applied in practice \citep{van2016deep,zhang2018human,huang2019autonomous}, and analyzed thoroughly in~\citep{one,two}. However, from a practical standpoint, double $Q$-learning employs a random switching mechanism between two $Q$-estimators to mitigate maximization bias. While this mechanism effectively reduces overestimation, 
it relies on an alternating update scheme between the two $Q$-estimators.
As a result, the overall learning process can theoretically take up to twice as long to converge under the same step-size settings~\citep{weng2020mean}.

Motivated by the aforementioned discussion, this paper proposes a modified double $Q$-learning called simultaneous double $Q$-learning (SDQ), which departs from the original in two key aspects: 1)
Elimination of random selection: It dispenses with the need for random selection between the two $Q$-estimators, a step that can slow down convergence in the original double $Q$-learning. This design replaces the stochastic estimator-selection mechanism in double $Q$-learning 
with a simultaneous update scheme, 
where both estimators are updated concurrently at each iteration. 
This eliminates the randomness in estimator selection, and the update structure becomes deterministic, which aligns naturally with the switching-system interpretation adopted in our theoretical framework.
2) Different roles of $Q$-estimators: 
In the original double $Q$-learning, the two estimators play asymmetric roles:
one estimator selects the greedy action based on its own values, 
while the other provides the target for the update. 
In contrast, SDQ introduces a cross-referenced mechanism in which 
each estimator uses the other to determine the greedy action,
but computes the target value using its own estimate. Specifically, the updates of SDQ can be expressed as:

\begin{align}
Q_{k+1}^{A}(s_{k},a_{k})&=Q_{k}^{A}(s_{k},a_{k})+\alpha_{k} \{ r_{k+1}\nonumber+\gamma Q_{k}^{A}(s_{k+1}, \text{argmax}_{a\in\mathcal{A}} Q_{k}^{B}(s_{k+1},a))\nonumber-Q_{k}^{A}(s_{k},a_{k})\},\nonumber\\
Q_{k+1}^{B}(s_{k},a_{k})&=Q_{k}^{B}(s_{k},a_{k})+\alpha_{k} \{ r_{k+1}+\gamma Q_{k}^{B}(s_{k+1}, \text{argmax}_{a\in \mathcal{A}} Q_{k}^{A}(s_{k+1},a))\nonumber-Q_{k}^{B}(s_{k},a_{k})\},\nonumber
\end{align}
where $Q_k^{A}$ and $Q_k^{B}$ denote two separate estimators of the optimal
action--value function $Q^*$ at iteration $k$.
The pair $(s_k,a_k)\in\mathcal{S}\times\mathcal{A}$ represents the
state--action pair sampled at time $k$, $r_{k+1}$ is the immediate reward
observed after taking action $a_k$ at state $s_k$, and $s_{k+1}$ is the
subsequent state.
The scalar $\alpha_k>0$ denotes the step size, and $\gamma\in(0,1)$ is the
discount factor.
Each estimator updates itself using the greedy action determined by the
other estimator, while evaluating the target value with its own estimate.
This mutual role exchange creates a symmetric interaction between the two estimators, and the resulting update equations form a coupled pair that can be naturally modeled as a discrete-time switching system \cite{liberzon2003switching}.
Such a symmetric formulation provides an analytical structure that facilitates finite-time convergence analysis.

To establish the finite-time error bounds, a novel analysis framework is developed in this paper. In particular, SDQ is modelled as a switching system~\citep{lee2020unified,lee2023discrete,lee2024final}, which captures the dynamics of double $Q$-learning as a discrete-time switching system model. For finite-time convergence analysis, two comparison systems -- termed the lower comparison system and the upper comparison system -- are derived to bound the behavior of the original switching system. Through convergence of these comparison systems, the following expected error bound is derived:
\begin{align}
\max\bigl\{\mathbb{E}\lVert Q_k^A - Q^* \rVert_\infty,\;
           \mathbb{E}\lVert Q_k^B - Q^* \rVert_\infty\bigr\}
&\le
  \frac{120\,\alpha^{1/2}\lvert \mathcal{S}\times \mathcal{A}\rvert}
       {d_{\min}^{9/2}(1-\gamma)^{11/2}}
  +\,\frac{48\,\rho^{\,k-4}k^{4}\lvert \mathcal{S}\times \mathcal{A}\rvert^{3/2}}
          {1-\gamma}\,,
\end{align}
where $\lvert \mathcal{S}\times \mathcal{A}\rvert$ is the number of the state-action pairs, $d_{\text{min}}$ is the minimum state-action occupation frequency, $\alpha\in(0,1)$ is the constant step-size and $\rho\coloneqq1-\alpha d_{\text{min}}(1-\gamma)\in(0,1)$
is the exponential decay rate.

Although the switching system model has been first introduced in~\cite{lee2020unified,lee2023discrete,lee2024final}, we extend this view to double $Q$-learning and provide a new finite-time analysis. 
We note that this extension is not trivial because the two estimators are coupled through their update rules. These additional dependencies complicate the finite-time analysis compared to standard $Q$-learning.
Therefore, the techniques used in the previous studies cannot be directly applied to double $Q$-learning. In this paper, new approaches have been developed to overcome this challenge. Details on the proposed analysis can be found in Section~\ref{sec:framework}, \ref{sec:analysis_process}. Finally, the main contributions are summarized as follows:
\begin{enumerate}
\item[(a)] 
SDQ is proposed to address maximization bias while exhibiting favorable convergence properties. Moreover, this modification enables double $Q$-learning to be viewed through the lens of a switching system and enables more efficient finite-time analysis.

\item[(b)] Based on the switching system model, novel finite-time analysis techniques and new expected error bounds are proposed for the SDQ. Moreover, the analysis frameworks introduced in this paper provide new theoretical perspectives and additional insights on double $Q$-learning and related algorithms. 
\end{enumerate}

\section{Related works}
There has been a growing body of research on finite-time analyses of 
$Q$-learning and its variants. Beyond double $Q$-learning, several recent studies 
have investigated finite-sample guarantees for vanilla $Q$-learning. 
For instance,~\cite{chen2024lyapunov} developed a Lyapunov-based theory for Markovian stochastic approximation. This theory provides finite-sample bounds for asynchronous tabular $Q$-learning under Markovian sampling.

Other works~\citep{qu2020finite, li2020sample} have also established 
non-asymptotic error bounds for $Q$-learning with similar Markovian 
sampling assumptions. 

There have been relatively few studies on the convergence analysis of double $Q$-learning. 
The first convergence proof was provided by~\citep{hasselt2010double}, 
but it only established asymptotic convergence. 
Finite-time convergence results have been more recently presented in~\citep{one,two}, 
which analyzed both synchronous and asynchronous double $Q$-learning 
under non-i.i.d. observation models. 
Unlike prior analyses of $Q$-learning and double $Q$-learning that typically rely on
Markovian sampling and cover-time conditions, our SDQ adopts an i.i.d.\ stochastic exploration assumption 
in which each state-action pair pair is independently drawn from a stationary distribution that is positive for all state-action pair.
This formulation simplifies the finite-time analysis while preserving the essential stochastic nature of RL. We note that the Markovian setting is more practical and realistic, as it captures temporal correlations commonly observed in RL environments.
However, with modest additional effort, such as incorporating mixing-time or cover-time conditions, the proposed framework can also be extended to the Markovian case as demonstrated in~\cite{lim2024finite}. Regarding the step-size, the previous analyses in~\citep{one,two} impose more restrictive ranges 
for convergence, whereas the proposed analysis allows a broader class of step-sizes $\alpha \in (0,1)$.

Moreover, existing works in~\citep{one,two} have primarily established high-probability error bounds, which provide probabilistic guarantees on the learning process.
In contrast, our analysis focuses on expected error bounds, which characterize the expected estimation accuracy and offer a complementary viewpoint.
While the two types of results serve different purposes, they are closely related:
an expected error bound can typically be converted into a probabilistic one through concentration inequalities.
Therefore, the expected bound used in this study should not be viewed as weaker or stronger, but rather as a complementary formulation that provides mean-error characterization aligned with our system-theoretic analysis framework.
In this sense, our result complements the existing probabilistic analyses and contributes to a more complete understanding of finite-time behavior in double $Q$-learning.

\section{Preliminaries}\label{sec:prelim}

\subsection{Markov decision problem}
We focus on an infinite-horizon discounted Markov decision process (MDP) in which an agent learns an optimal policy by maximizing the expected discounted sum of future rewards through sequential interactions with the environment. 
The environment is modeled by a finite state space ${\cal S}:=\{ 1,2,\ldots ,|{\cal S}|\}$ and a finite action space ${\cal A}:= \{1,2,\ldots,|{\cal A}|\}$, where $|{\cal S}|$ and $|{\cal A}|$ denote the cardinalities of the state and action spaces, respectively.
At each step, given the current state $s\in{\cal S}$, the agent chooses an action $a\in{\cal A}$, and the system transitions to a next state $s'\in{\cal S}$ with probability $P(s' \mid s,a)$. It receives a reward $r(s,a,s')$. For simplicity, we assume that the reward function is deterministic and denote it by 
$r(s_k, a_k, s_{k+1}) =: r_{k+1}$, where $k \in \{0,1,\ldots\}$. A \emph{deterministic policy} $\pi: \mathcal{S} \to \mathcal{A}$ assigns to each state 
$s \in \mathcal{S}$ a specific action $\pi(s) \in \mathcal{A}$. 
The objective of the Markov decision problem is to determine an 
\emph{optimal policy} $\pi^*$ that maximizes the expected cumulative discounted rewards 
over an infinite horizon:
\begin{align*}
\pi^* := \arg\max_{\pi \in \Theta} 
\mathbb{E}\!\left[\sum_{k=0}^{\infty} \gamma^k r_{k+1} \,\bigg|\, \pi\right],
\end{align*}
where $\gamma \in [0,1)$ is the discount factor, 
$\Theta$ denotes the set of all admissible deterministic policies, 
$(s_0,a_0,s_1,a_1,\ldots)$ represents a state--action trajectory generated under policy $\pi$, 
and $\mathbb{E}[\cdot|\pi]$ indicates the expectation conditioned on $\pi$. The $Q$-function associated with a policy $\pi$ is defined as
\begin{align*}
Q^{\pi}(s,a)
= \mathbb{E}\!\left[\sum_{k=0}^{\infty} \gamma^k r_{k+1} 
\,\bigg|\, s_0=s,\, a_0=a,\, \pi \right], 
\quad (s,a) \in \mathcal{S} \times \mathcal{A}.
\end{align*}
The optimal $Q$-function is given by 
$Q^*(s,a) = Q^{\pi^*}(s,a)$ for all $(s,a) \in \mathcal{S} \times \mathcal{A}$. 
Once $Q^*$ is obtained, the optimal policy can be recovered via the greedy rule:
\begin{align*}
\pi^*(s) = \arg\max_{a \in \mathcal{A}} Q^*(s,a).
\end{align*}
Throughout the paper, we assume that the MDP is ergodic. It ensures the existence of a stationary state distribution and the well-posedness of the problem.

\subsection{Switching system }
Following standard notions in control theory~\citep{liberzon2003switching,lin2009stability,khalil2002nonlinear}, 
a discrete-time switching system can be regarded as a particular instance of a nonlinear dynamical system.
We briefly revisit this concept, as it forms the analytical foundation for representing the update mechanism of $Q$-learning.
We begin with a general nonlinear discrete-time system:
\begin{align}
x_{k+1} = f(x_k), \quad x_0 = z \in \mathbb{R}^n, \quad k \in \{1,2,\ldots\}, \label{eq:nonlinear-system}
\end{align}
where $x_k \in \mathbb{R}^n$ denotes the system state and $f:\mathbb{R}^n \to \mathbb{R}^n$ is a nonlinear mapping. 
A point $x^* \in \mathbb{R}^n$ is called an equilibrium point of~\eqref{eq:nonlinear-system} if the state remains at $x^*$ whenever the system starts from $x_0 = x^*$. 
For \eqref{eq:nonlinear-system}, equilibrium points are the real roots of the equation $f(x) = x$. 
Moreover, an equilibrium $x^*$ is said to be globally asymptotically stable if, for any initial condition $x_0 \in \mathbb{R}^n$, the state trajectory satisfies $x_k \to x^*$ as $k \to \infty$.

A subclass of nonlinear systems is the \emph{linear switching system} \cite{liberzon2003switching}, expressed as
\begin{align}
x_{k+1} = A_{\sigma_k} x_k, \quad x_0 = z \in \mathbb{R}^n, \quad k \in \{0,1,\ldots\}, \label{eq:switched-system}
\end{align}
where $x_k \in \mathbb{R}^n$ is the state, 
$\sigma_k \in \mathcal{M} := \{1,2,\ldots,M\}$ denotes the mode at time $k$, 
and $\{A_\sigma\}_{\sigma \in \mathcal{M}}$ are the subsystem matrices. 
The switching signal $\sigma_k$ may vary arbitrarily or follow a prescribed policy, such as a state-feedback rule $\sigma_k = \sigma(x_k)$. 
A more general formulation is the \emph{affine switching system}:
\begin{align*}
x_{k+1} = A_{\sigma_k} x_k + b_{\sigma_k}, 
\quad x_0 = z \in \mathbb{R}^n, \quad k \in \{0,1,\ldots\},
\end{align*}
where $b_{\sigma_k} \in \mathbb{R}^n$ represents a mode-dependent additional input vector. 
The presence of this additional affine term generally increases the difficulty of ensuring system stability.

\subsection{Double Q-learning}
Double $Q$-learning~\citep{hasselt2010double} is a variant of $Q$-learning~\citep{watkins}, which can reduce the maximization bias in its update by updating one of the two $Q$-estimators $Q^{A}_k$ and $Q^{B}_k$, which is selected randomly. Therefore, the corresponding update can be presented as follows: 
\begin{align}\label{eqn:original_update}
Q_{k+1}^{A}(s_{k},a_{k})&=\zeta_{k} Q_{k}^{A}(s_{k},a_{k})+\alpha_{k}\zeta_{k} \{ r_{k+1}\nonumber+\gamma Q_{k}^{B}(s_{k+1}, \text{argmax}_{a\in \mathcal{A}} Q_{k}^{A}(s_{k+1},a))\nonumber-Q_{k}^{A}(s_{k},a_{k})\},\nonumber\\
Q_{k+1}^{B}(s_{k},a_{k})&=(1-\zeta_{k}) Q_{k}^{B}(s_{k},a_{k})+\alpha_{k}(1-\zeta_{k}) \{ r_{k+1}+\gamma Q_{k}^{A}(s_{k+1}, \text{argmax}_{a\in \mathcal{A}} Q_{k}^{B}(s_{k+1},a))-Q_{k}^{B}(s_{k},a_{k})\},
\end{align}
where $Q_k^{A}$ and $Q_k^{B}$ denote two separate estimators of the optimal
action--value function $Q^*$ at iteration $k$.
The pair $(s_k,a_k)\in\mathcal{S}\times\mathcal{A}$ represents the
state--action pair sampled at time $k$, $r_{k+1}$ is the immediate reward
observed after taking action $a_k$ at state $s_k$, and $s_{k+1}$ is the
subsequent state.
The scalar $\alpha_k>0$ denotes the step size at iteration $k$, and
$\gamma\in(0,1)$ is the discount factor.
The Bernoulli random variable $\zeta_k\in\{0,1\}$ determines which estimator
is updated at iteration $k$, with
$\mathbb{P}(\zeta_k=0)=\mathbb{P}(\zeta_k=1)=0.5$.
At each iteration, only one of the two estimators is updated using the
greedy action determined by the other estimator. By eliminating the max operator in its updates, it is known to reduce effectively the maximization bias.

\subsection{Assumption and Definition}\label{sec:ass_def}
Throughout, we make the following standard assumptions, which are widely adopted in the RL literature.
We consider the scenario where the data samples are generated from an RL agent interacting with the environment under a fixed
behavior policy $\beta$. At each iteration, the state--action pair $(s,a)$ is assumed to be drawn independently from the stationary state distribution $p$ and the behavior policy $\beta$, 
which leads to the following joint distribution:
\[
d(s,a)=p(s)\beta (a \vert s), \quad (s,a) \in \mathcal{S} \times \mathcal{A}.
\]
Moreover, the following assumptions will be adopted. 
\begin{assumption}\label{asm:one}
(Sufficient exploration) 
$d(s,a)>0$ for all $s\in \mathcal{S}, a\in \mathcal{A}$.
\end{assumption}
\begin{assumption}(Constant step-size)\label{asm:two}
The step-size is a constant $\alpha\in (0,1).$
\end{assumption}
\begin{assumption}(Unit bounded reward)\label{asm:three}
We have
\[
\max_{(s,a,s')\in \mathcal{S} \times \mathcal{A} \times \mathcal{S}} \lvert r(s,a,s') \rvert = R_{\text{max}} \leq 1. 
\]
\end{assumption}
\begin{assumption}(Bounded initialization)\label{asm:four}
The initial iterate $Q_{0}$ satisfies $\lVert Q_{0} \rVert_{\infty}\leq 1.$
\end{assumption}
\noindent\cref{asm:one} 
guarantees sufficient coverage of the state-action space
and~\cref{asm:three} and~\ref{asm:four} are introduced without loss of generality and for simplicity of the analysis.
For notational convenience, we define the following quantities that will be used throughout the paper.
\begin{definition}
\label{def:d_max_and_d_min}$\,$
1) Maximum state-action occupancy frequency:
\begin{align*}
d_{\mathrm{max}}\coloneq\max_{(s,a)\in \mathcal{S} \times \mathcal{A}}d(s,a)\in(0,1).
\end{align*}
2) Minimum state-action occupancy frequency:
\begin{align*}
d_{\mathrm{min}}\coloneq\min_{(s,a)\in \mathcal{S} \times \mathcal{A}}d(s,a)\in(0,1).
\end{align*}
3) Exponential decay rate:
\begin{align}
\rho\coloneq1-\alpha d_{\mathrm{min}}(1-\gamma).
\end{align}
\end{definition}
\noindent Under~\cref{asm:two}, the decay rate satisfies $\rho \in (0,1)$.
Throughout the paper, we will use the following compact notations for dynamical system representations:
\[
P=
\begin{bmatrix}
P_{1}\\
\vdots\\
P_{\lvert \mathcal{A} \rvert}
\end{bmatrix},
R=
\begin{bmatrix}
R_{1}\\
\vdots\\
R_{\lvert \mathcal{A} \rvert}
\end{bmatrix},
Q=
\begin{bmatrix}
Q(\cdot,1)\\
\vdots\\
Q(\cdot,\lvert \mathcal{A} \rvert)
\end{bmatrix},
\]
and
\[
D_a=\mathrm{diag}\big(d(1,a),\ldots,d(|\mathcal{S}|,a)\big),\quad
D=\mathrm{blkdiag}(D_1,\ldots,D_{|\mathcal{A}|}).
\]
where $P_a = P(\cdot|a,\cdot) \in \mathbb{R}^{|\mathcal{S}|\times|\mathcal{S}|}$, $Q(\cdot,a)\in\mathbb{R}^{|\mathcal{S}|}$ for $a\in\mathcal{A}$, and $R_a(s)\coloneqq\mathbb{E}[r(s,a,s')|s,a]$.
Here, $\mathrm{diag}(\cdot)$ denotes a diagonal matrix formed from its vector
arguments, and $\mathrm{blkdiag}(\cdot)$ denotes a block-diagonal matrix whose
diagonal blocks are the given matrices.
Note that $P\in\mathbb{R}^{|\mathcal{S}||\mathcal{A}|\times|\mathcal{S}|}$, $R,Q\in\mathbb{R}^{|\mathcal{S}||\mathcal{A}|}$, and $D\in\mathbb{R}^{|\mathcal{S}||\mathcal{A}|\times|\mathcal{S}||\mathcal{A}|}$.
With this notation, the $Q$-function can be represented as a single stacked vector $Q\in\mathbb{R}^{|\mathcal{S}||\mathcal{A}|}$ that enumerates all $Q(s,a)$ values for every $(s,a)\in\mathcal{S}\times\mathcal{A}$.
Each entry $Q(s,a)$ can be expressed as $Q(s,a) = (e_a \otimes e_s)^T Q$, 
where
$e_s\in\mathbb{R}^{|\mathcal{S}|}$ and $e_a\in\mathbb{R}^{|\mathcal{A}|}$
denote the standard basis vectors, whose $s$-th and $a$-th components are
equal to one and all other components are zero, respectively,
and $\otimes$ denotes the Kronecker product.
Under~\cref{asm:two}, the matrix $D$ is a nonsingular diagonal matrix with strictly positive diagonal elements.
For any stochastic policy $\pi:\mathcal{S}\to\Delta_{|\mathcal{A}|}$, where $\Delta_{|\mathcal{A}|}$ denotes the probability simplex over $\mathcal{A}$, we define the matrix

\begin{align}
\Pi^{\pi}\coloneqq
\begin{bmatrix}
\pi(1)^{T}\otimes e_{1}^{T}\\
\pi(2)^{T}\otimes e_{2}^{T}\\
\vdots\\
\pi(\lvert S \rvert)^{T}\otimes e_{\lvert \mathcal{S} \rvert}^{T}\\
\end{bmatrix}
\in \mathbb{R}^{\lvert \mathcal{S}\rvert\times\lvert \mathcal{S}\rvert\lvert \mathcal{A}\rvert}\label{eq:swtiching-matrix}
\end{align}
It is well known that $P\Pi^\pi \in \mathbb{R}^{|\mathcal{S}||\mathcal{A}|\times|\mathcal{S}||\mathcal{A}|}$ represents the transition probability matrix of state–action pairs under policy $\pi$.
In the case of a deterministic policy $\pi:\mathcal{S}\to\mathcal{A}$, the stochastic policy can be equivalently expressed using a one-hot encoding vector
$\vec{\pi}(s) := e_{\pi(s)} \in \Delta_{|\mathcal{A}|}$.
The resulting action-transition matrix takes the same form as~(\ref{eq:swtiching-matrix}), with $\pi$ replaced by $\vec{\pi}$.
For any $Q \in \mathbb{R}^{|\mathcal{S}||\mathcal{A}|}$, we denote by $\pi_Q(s) := \arg\max_{a\in\mathcal{A}} Q(s,a)$ the greedy policy with respect to $Q$, and use the shorthand notation $\Pi_Q := \Pi^{\pi_Q}$.

We recall a standard result ensuring that the $Q$-learning sequence remains bounded~\citep{gosavi2006boundedness}, 
which plays an important role in our analysis.
\begin{lemma} \citep{gosavi2006boundedness}\label{lem:lem1}
If the step-size is less than one, then for all $k\geq 0$
\[
\lVert Q_{k} \rVert_{\infty}\leq Q_{\mathrm{max}} = \frac{\mathrm{max}{\{R_{\mathrm{max}},\mathrm{max}_{(s,a)\in \mathcal{S} \times \mathcal{A}}\|Q_{0}(s,a)\|_{\infty}\}}}{1-\gamma}.
\]
\end{lemma}
\noindent From Assumptions~\ref{asm:three} and~\ref{asm:four}, we can easily see that $Q_{\max}\leq\frac{1}{1-\gamma}$.

\section{Simultaneous double Q-learning (SDQ)}\label{sec:MDQL}
\label{headings}
\subsection{Algorithm}\label{sec:MDQL_alg}
In this paper, we consider the following modified double $Q$-learning, called simultaneous double $Q$-learning (SDQ):
\begin{align}
\label{eqn:modified_update}
Q_{k+1}^{A}(s_{k},a_{k})&=Q_{k}^{A}(s_{k},a_{k})+\alpha_{k} \{ r_{k+1}\nonumber+\gamma Q_{k}^{A}(s_{k+1}, \text{argmax}_{a\in A} Q_{k}^{B}(s_{k+1},a))\nonumber-Q_{k}^{A}(s_{k},a_{k})\},\nonumber\\
Q_{k+1}^{B}(s_{k},a_{k})&=Q_{k}^{B}(s_{k},a_{k})+\alpha_{k} \{ r_{k+1}+\gamma Q_{k}^{B}(s_{k+1}, \text{argmax}_{a\in A} Q_{k}^{A}(s_{k+1},a))-Q_{k}^{B}(s_{k},a_{k})\},
\end{align}
where $Q_k^{A}$ and $Q_k^{B}$ denote two separate estimators of the optimal
action--value function $Q^*$ at iteration $k$.
The pair $(s_k,a_k)\in\mathcal{S}\times\mathcal{A}$ represents the
state--action pair sampled at time $k$, $r_{k+1}$ is the immediate reward
observed after taking action $a_k$ at state $s_k$, and $s_{k+1}$ is the
subsequent state.
The scalar $\alpha_k>0$ denotes the step size at iteration $k$, and
$\gamma\in(0,1)$ is the discount factor.
The first difference between the original double $Q$-learning and SDQ is the role of each $Q$-estimator in the update. In the original double $Q$-learning, an optimal action is selected from the same $Q$-estimator, and it employs the other $Q$-estimator for bootstrapping. On the other hand, in the proposed version, an optimal action is selected from the other $Q$-estimator, and it employs the same $Q$-estimator for bootstrapping. This modification enables the use of the switching system framework from~\cite{lee2023discrete}.
It overcomes the difficulty caused by the switched order of $Q^A_k$ and $Q^B_k$ in the original double $Q$-learning while retaining the advantage of reducing overestimation bias.

The other difference is in the Bernoulli variable. Unlike the standard double $Q$-learning, which uses a Bernoulli variable for the $Q$-estimator selection, the modified version updates the two $Q$-estimators synchronously, which can potentially speed up the convergence. However, we note that our analysis can also include the Bernoulli random selection as in the original form without major changes in the finite-time error analysis. Besides, a potential issue that arises by eliminating the random $Q$-estimator selection is that if initially $Q^A_0 = Q^B_0$, then $Q^A_k = Q^B_k$ for all $k \ge 0$. In this case,~(\ref{eqn:modified_update}) is reduced to the standard $Q$-learning because in this case, $Q^A_k = Q^B_k$ for all $k\geq 0$. To bypass the issue for implementation, one simple approach is to randomly initialize $Q^A_0$ and $Q^B_0$ so that $Q^A_0 \neq Q^B_0$.

To demonstrate its effect, let us consider the example in Figure~\ref{fig:maximization_bias}(left) (adopted from \citep{sutton2018reinforcement}, \text{Ch. 6.7}). We consider an epsilon-greedy exploration with $\epsilon = 0.1$, constant step‐size $\alpha = 0.1$, and discount factor $\gamma=0.9$. The experiment consists of 1,000 independent runs, each comprising 500 training episodes. The initial $Q$‐values for SDQ are uniformly sampled from $[-0.3,0.3]$. We also include a perturbed $Q$‐learning variant, in which the $Q$‐table is initialized by sampling from the same uniform distribution $[-0.3,0.3]$. As observed in the result, SDQ initially suffers from overestimation of $Q$‐values due to the random initialization of its two estimators. However, this bias is quickly mitigated, and SDQ converges at a rate similar to that of original double $Q$‐learning, as shown in Figure~\ref{fig:maximization_bias}(right). Furthermore, both standard $Q$‐learning and the perturbed $Q$‐learning variant continue to exhibit overestimation, highlighting the efficacy of SDQ in mitigating this bias.
\begin{figure*}[h!]
  \centering
    {\includegraphics[scale=0.25]{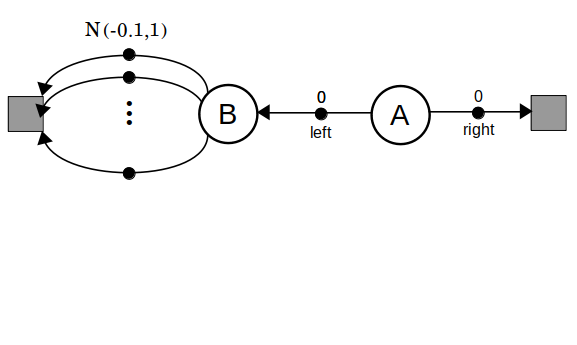}}
    {\includegraphics[scale=0.38]{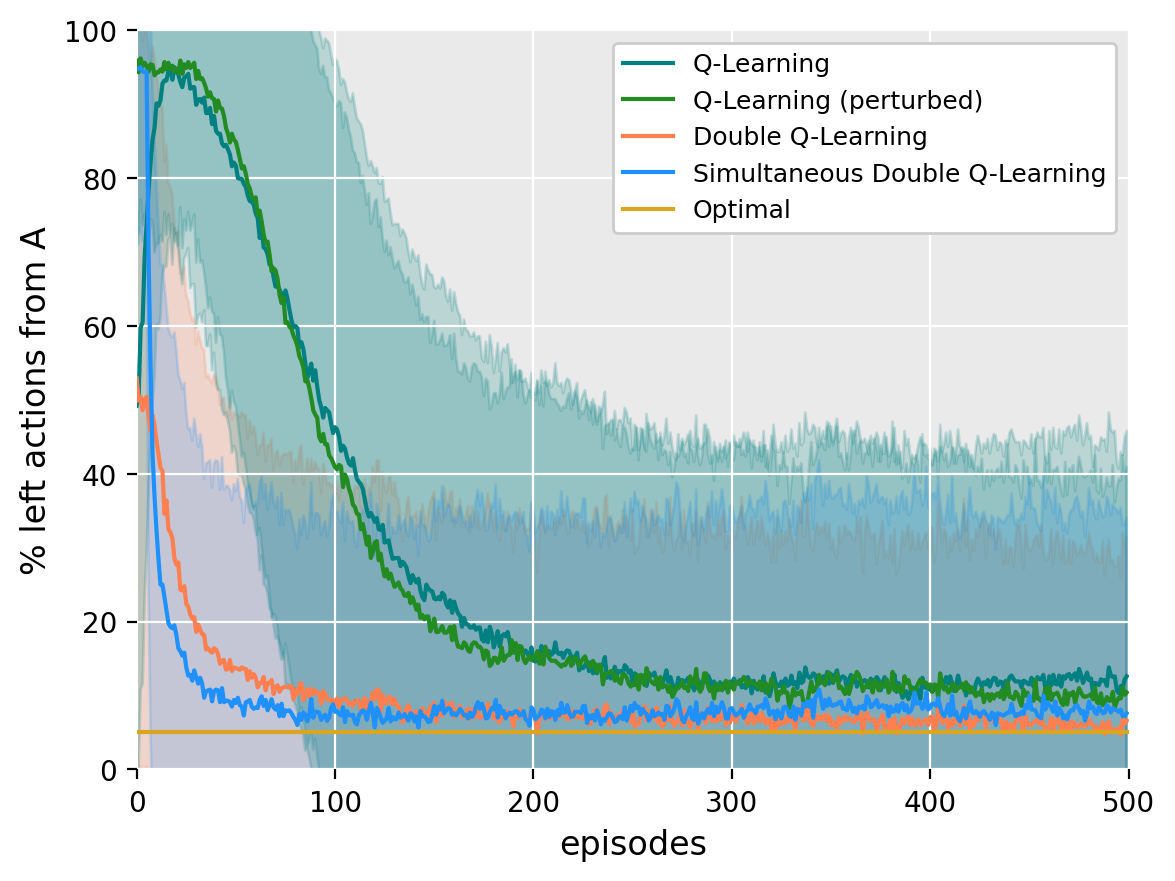}} 
    
  \caption{\textbf{Left}: An example from~\citep{sutton2018reinforcement}. The episode always starts from the $A$ node. Taking the right action from the $A$ node results in zero reward, and the episode is terminated. Otherwise, taking the left action leads to state $B$, where the agent chooses one of 10 available actions. Executing any of these actions results in a reward sampled from a normal distribution with mean $-0.1$ and standard deviation $1$. Then, the episode is terminated as well. Although $Q^{*}(A,\text{right})$ is zero and $Q^{*}(A,\text{left})$ is $-0.1\gamma$, $Q$-learning favors left action because of maximization bias. \textbf{Right}: Comparison of experiment results: SDQ vs. double $Q$-learning vs. $Q$-learning vs. $Q$-learning (perturbed, with randomly initialized $Q$-values).} 
  \label{fig:maximization_bias}
\end{figure*}

\subsection{Experiment}\label{sec:experiment}
We organize our evaluation into two complementary studies. First, we test 
the ability of SDQ to correct maximization bias in a simple stochastic 8×8 grid world where each step yields a stochastic reward. This environment makes overestimation bias clear and allows us to compare SDQ against both standard and perturbed $Q$-learning and double $Q$-learning. Here, the perturbed $Q$-learning variant refers to standard $Q$-learning
with randomly initialized $Q$-values. Second, we demonstrate that SDQ converges faster than double $Q$-learning across three deterministic OpenAI Gym tasks, FrozenLake-v0, CliffWalking-v0, and Taxi-v3.
All agents use the same epsilon-greedy exploration strategy, learning rate, and discount factor.
Each agent is trained until its learning curve has fully stabilized. Together, these experiments show that SDQ not only nearly eliminates overestimation bias but also delivers consistent gains in convergence speed.
\subsubsection{Grid World}

We begin by evaluating SDQ in a simple stochastic grid‐world from Figure~\ref{fig:grid_world}(left) designed to expose maximization bias. The agent occupies an 8×8 grid, starting in the lower‐left cell and seeking the upper‐right goal. Each non‐terminal transition yields a reward of $-10$ or $+2$ with equal probability, while entering the goal state grants $+20$ and immediately terminates the episode.  All five algorithms, $Q$-learning, perturbed $Q$-learning, double $Q$-learning, perturbed double $Q$-learning, and SDQ, are run for 10,000 steps using epsilon-greedy where $\varepsilon(s)=1/\sqrt{n(s)}$ and $n(s)$ is the number of times state $s$ has been visited. The learning rate $\alpha_k(s, a)$ is chosen as a linear decay, $\alpha_k(s, a) = 1/n_k(s, a)$. In the case of double $Q$-learning, the count $n_k(s, a)$ is set to $n_k^A(s, a)$ when updating $Q_k^A$, and to $n_k^B(s, a)$ when updating $Q_k^B$. The variables $n_k^A$ and $n_k^B$ respectively record how many times each state–action pair has been updated in the two value functions. The discount factor is set to $\gamma=0.95$, and $Q$‐values are initialized by sampling uniformly from $[-0.3,0.3]$, consistent with the setup in Figure \ref{fig:maximization_bias}. Results are averaged over 10 independent runs. Figure~\ref{fig:grid_world}(middle) shows the average cumulative reward per step. SDQ achieves a marginally higher final return than the other methods. Figure~\ref{fig:grid_world}(right) plots the maximum
action-value at the start state,
$\max_{a} Q_k(s_0,a)$, against the true optimal value
$\max_{a} Q^{*}(s_0,a)$ (dashed line). Both standard and perturbed $Q$-learning clearly overestimate, while the two double $Q$-learning variants underestimate. SDQ stays closest to the optimum throughout.
This result shows that it effectively corrects the overestimation bias.

\subsubsection{FrozenLake, CliffWalking, and Taxi Environments}
Next, we evaluate SDQ on three deterministic Gym tasks, FrozenLake-v0, CliffWalking-v0, and Taxi-v3, and show that it consistently converges faster than double $Q$-learning. In these experiments, we compare SDQ against original double $Q$-learning and a perturbed variant. We employ epsilon-greedy exploration with \(\epsilon=0.1\), a constant learning rate \(\alpha=0.01\), and a discount factor \(\gamma=0.99\) for all algorithms. $Q$-estimators for SDQ are initialized randomly with the uniform distribution $[0,0.01]$, and for the perturbed double $Q$-learning variant we apply the same uniform initialization $[0,0.01]$ to both estimators. To account for the sparse rewards in FrozenLake-v0, we evaluate the average episodic reward after applying a moving window of size 100 for each episode, smoothing the reward signal since the agent only receives a $+1$ reward upon successful completion. Agents are trained for 10,000 episodes in FrozenLake-v0 to accommodate its sparse rewards, for 500 episodes in CliffWalking-v0, and for 30,000 episodes in Taxi-v3. These episode counts ensure that each environment’s learning curve has stabilized. For all experiments, the results are averaged over $30$ independent runs.

Figure~\ref{fig:experiment} shows that SDQ achieves a modest but consistent improvement in convergence speed over both standard double $Q$-learning and its perturbed variant, suggesting this gain stems from its structural design rather than initialization alone. 
While the final returns are comparable to those of double $Q$-learning,
SDQ generally reaches its steady-state performance faster, which is consistent with the theoretical insight on its improved stability.

\begin{figure}[tpb]
  \centering
  \begin{minipage}{0.3\textwidth}
    \centering
    \begin{tikzpicture}[scale=0.33]
        \draw[step=1cm,thick] (0,0) grid (8,8);
        \node at (0.5,0.5) {\textbf{S}};
        \node at (7.5,7.5) {\textbf{G}};
    \end{tikzpicture}
  \end{minipage}
  \hspace{0.03\textwidth} 
  \begin{minipage}{0.3\textwidth}
    \centering
    \includegraphics[width=\linewidth]{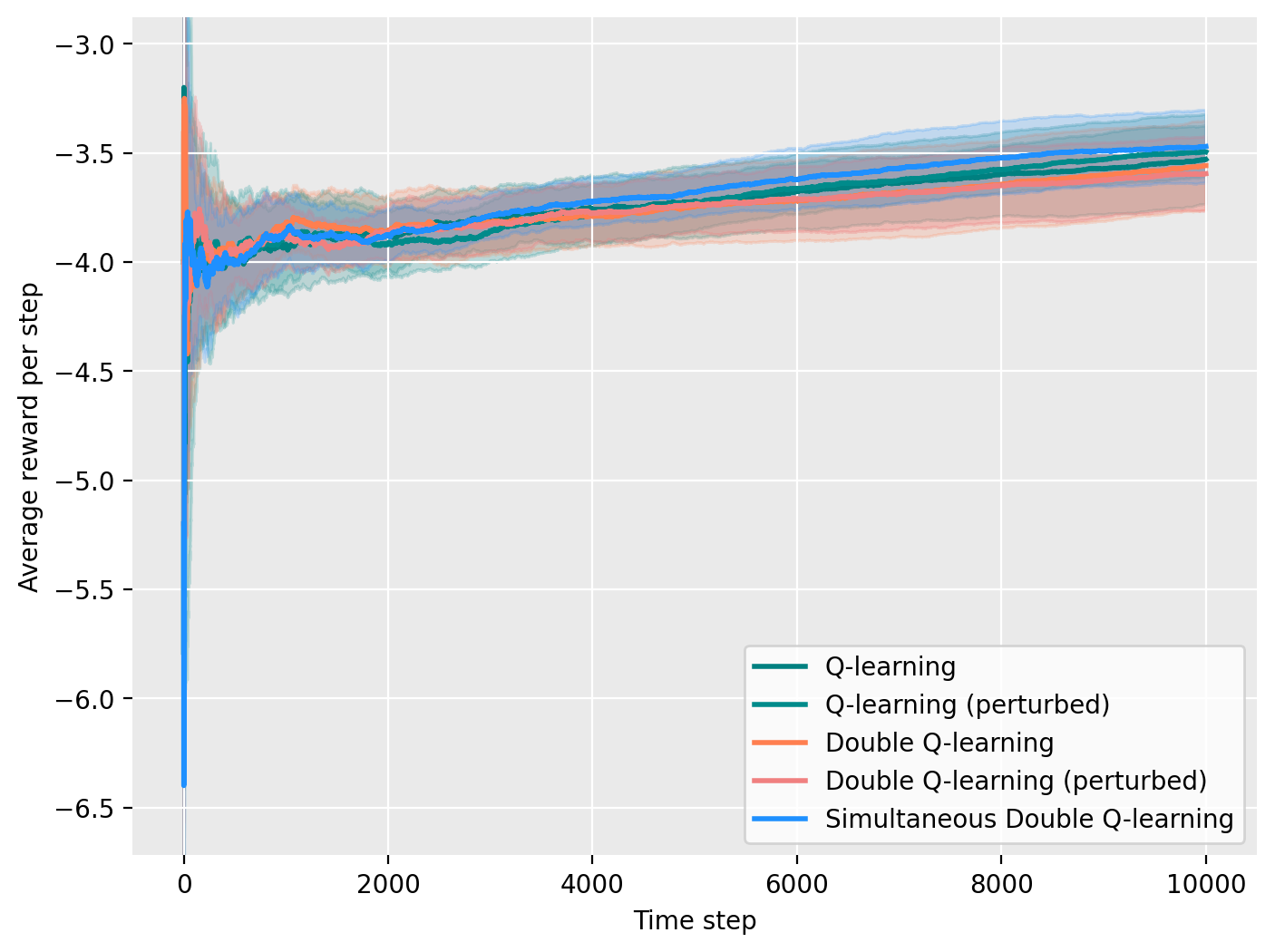}
  \end{minipage}
  \hspace{0.03\textwidth}
  \begin{minipage}{0.3\textwidth}
    \centering
    \includegraphics[width=\linewidth]{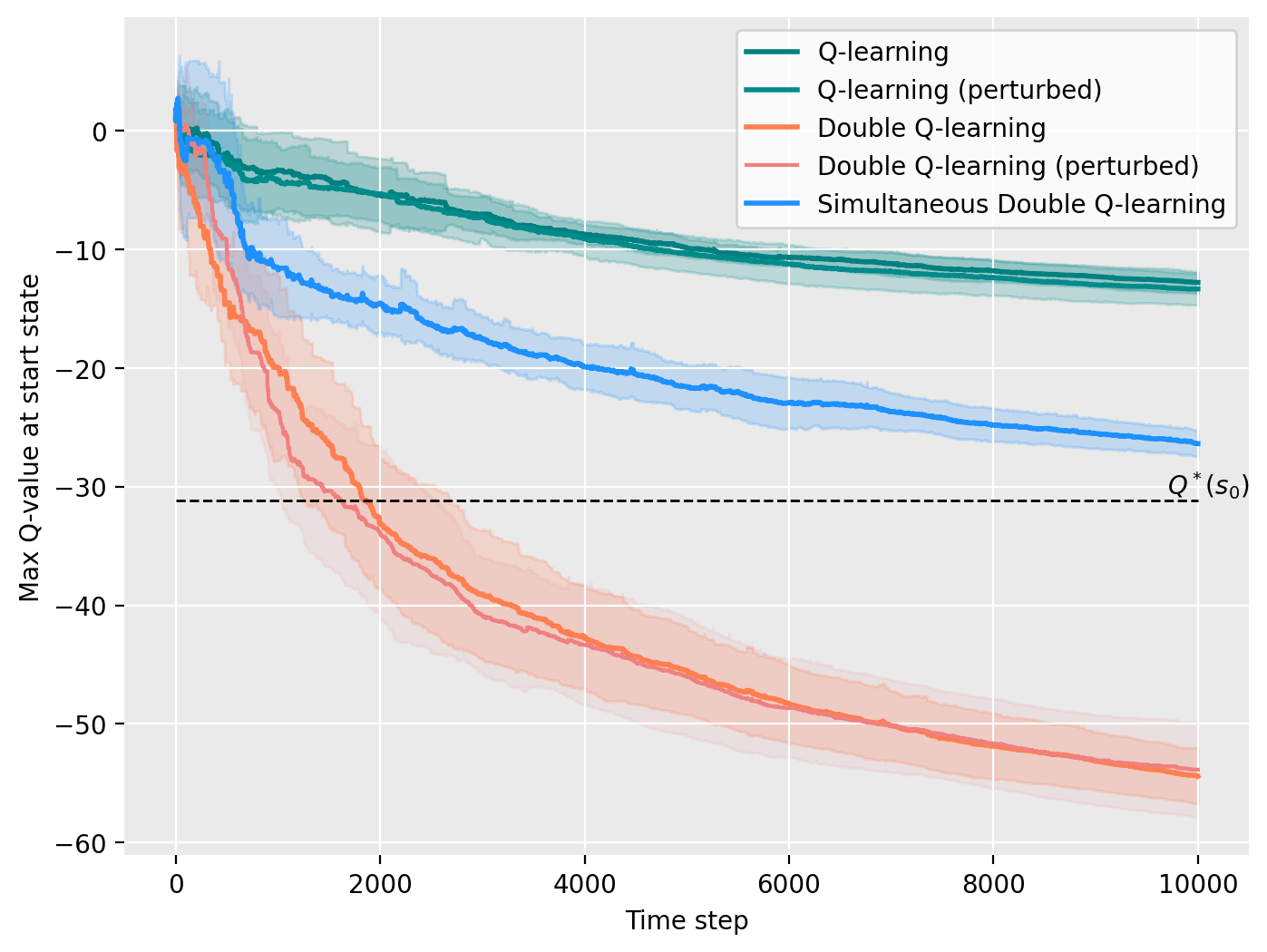}
  \end{minipage}
  \caption{\textbf{Left}: 8×8 Grid world example. \textbf{Middle}: Average cumulative reward per step for each algorithm. \textbf{Right}: Evolution over time of the start‐state’s maximum action‐value.}
  \label{fig:grid_world}
\end{figure}

\begin{figure*}[h!]
\centering
\begin{subfigure}{0.3\textwidth}
    \includegraphics[width=\textwidth]{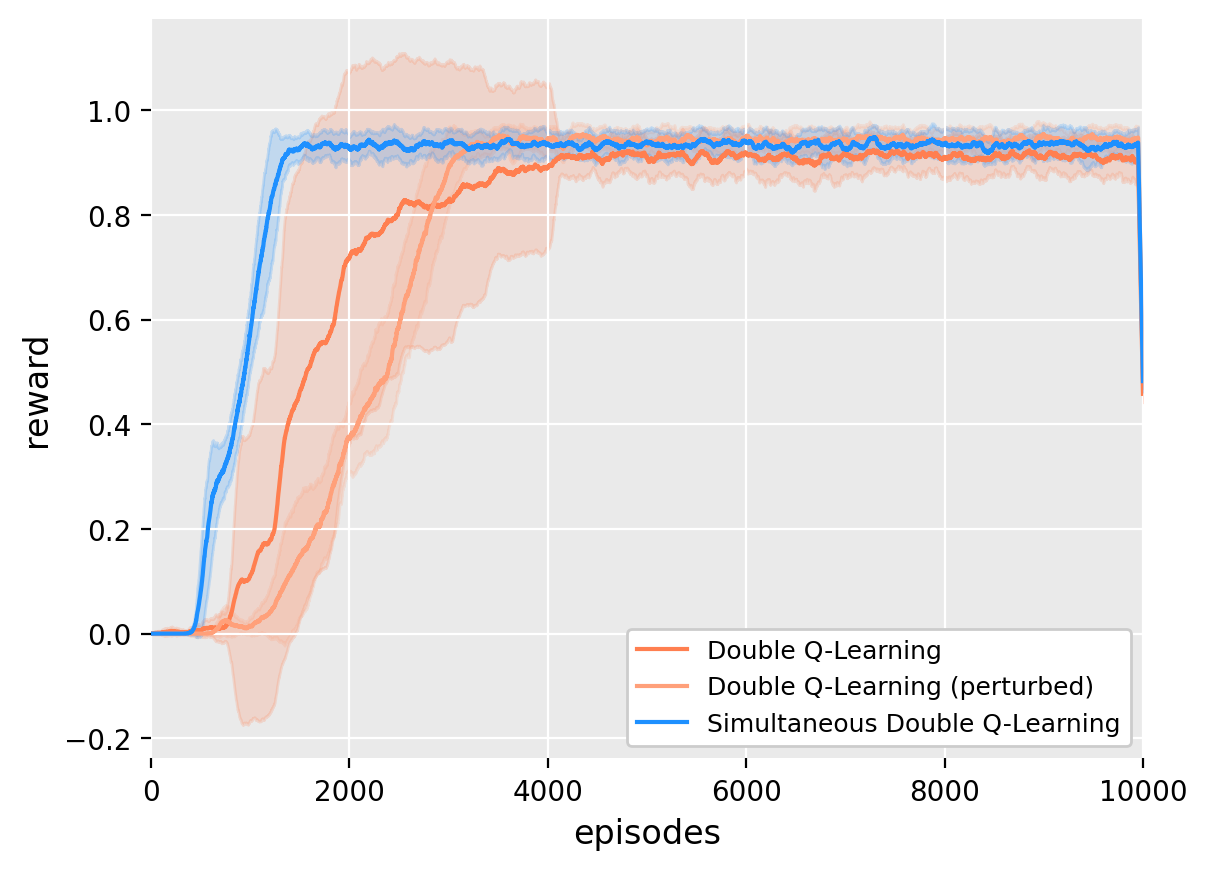}
    \caption{FrozenLake-v0}
    \label{fig:first}
\end{subfigure}
\hfill
\begin{subfigure}{0.3\textwidth}
    \includegraphics[width=\textwidth]{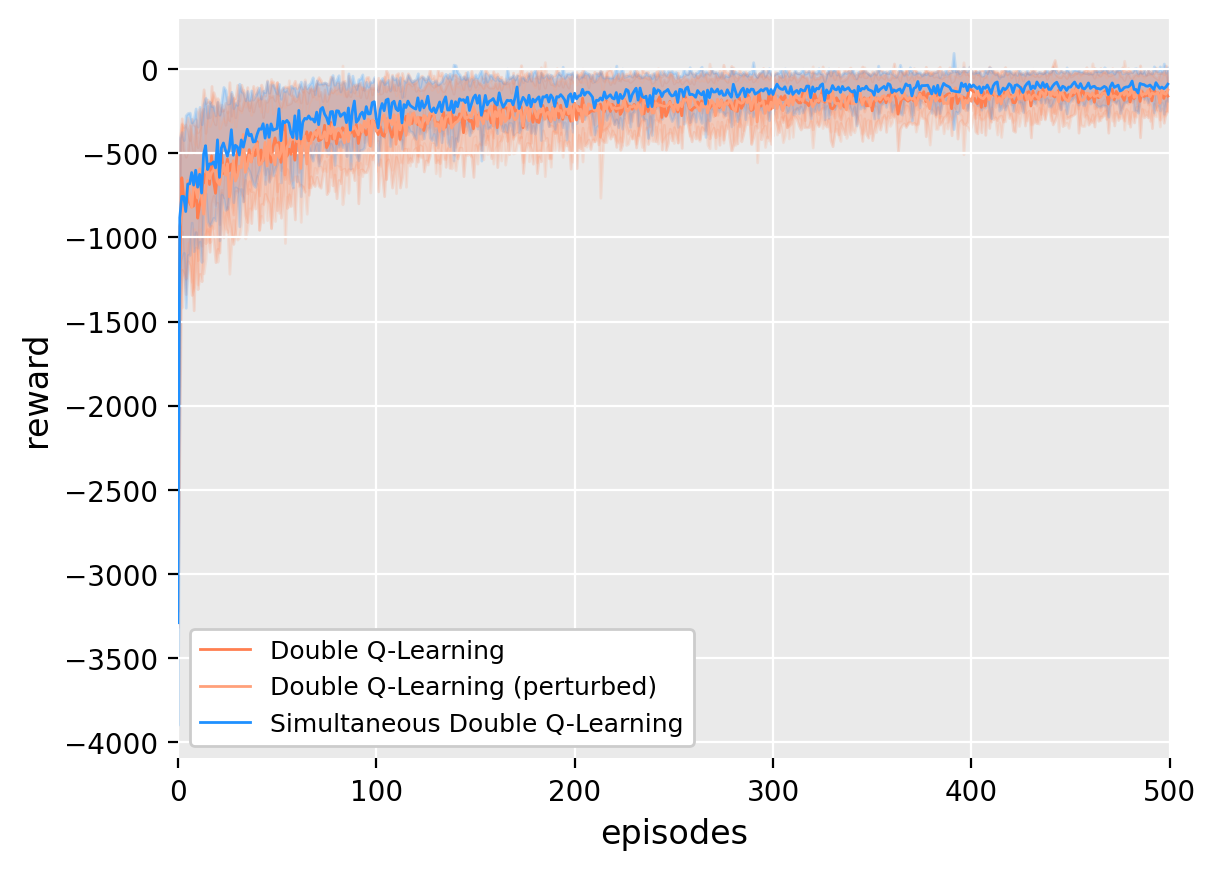}
    \caption{CliffWalking-v0}
    \label{fig:second}
\end{subfigure}
\hfill
\begin{subfigure}{0.3\textwidth}
    \includegraphics[width=\textwidth]{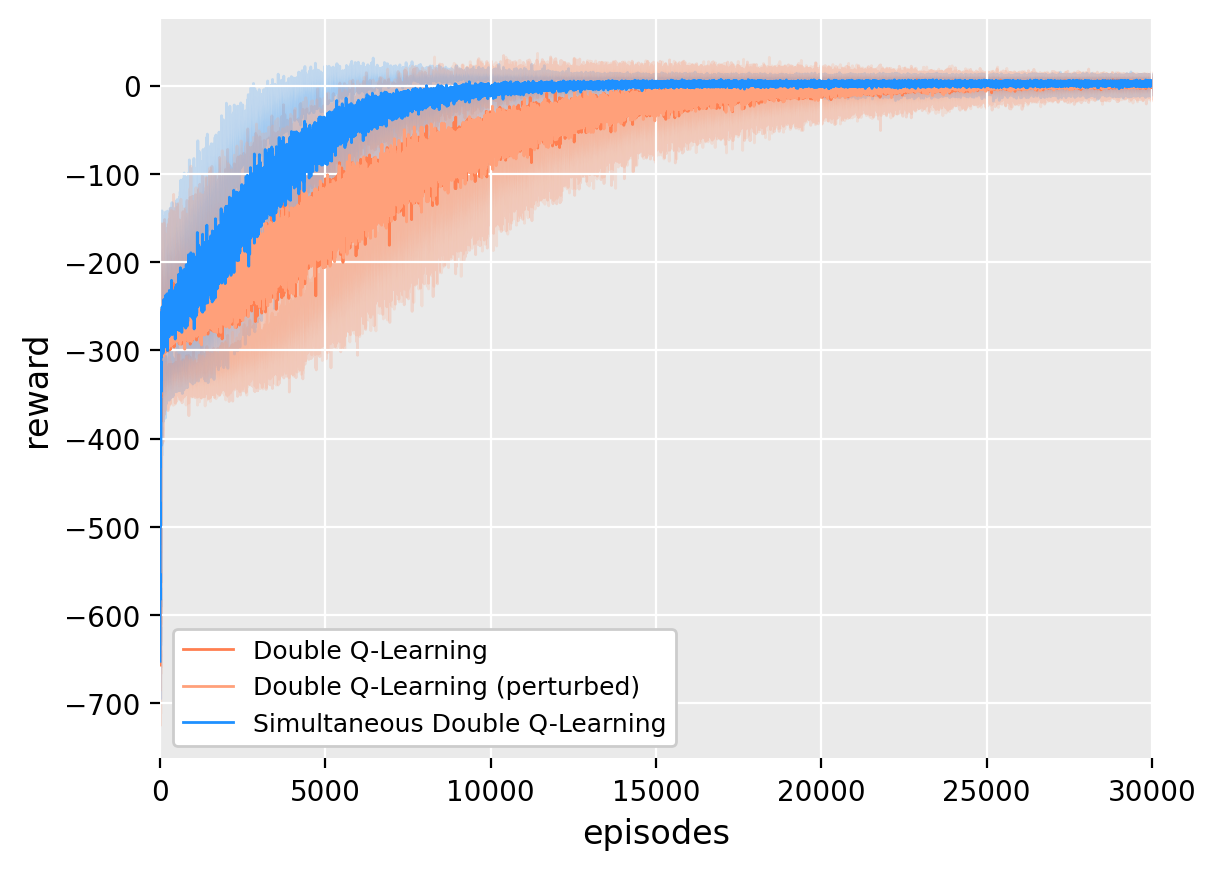}
    \caption{Taxi-v3}
    \label{fig:third}
\end{subfigure}
        
\caption{Comparison of experiment results: SDQ vs. double $Q$-learning vs. double $Q$-learning (perturbed, with randomly initialized $Q$-values). }
\label{fig:experiment}
\end{figure*}

\begin{remark}[Applicability to complex environments]
Recent studies have extended the double-estimator framework of double $Q$-learning 
to more complex and high-dimensional domains, including continuous-control and real-world dynamic settings 
(e.g.,~\citep{nagarajan2025double, patil2025optimizing, fan2025incremental}). 
These works demonstrate that the double-estimator structure remains a useful foundation 
for achieving stable and adaptive learning in complex environments. 
Unlike these approaches, which typically employ stochastic or alternating estimator updates, 
our SDQ adopts a deterministic coupling mechanism where both estimators are updated concurrently. 
This structural difference enables a tractable finite-time convergence
analysis within a control-theoretic framework, 
clarifying the theoretical role of estimator coupling beyond its empirical advantages.\end{remark}

\subsection{Finite-time error bounds}
In this subsection, we present finite-time error bounds for SDQ. Through the analysis given in this paper, we can derive a finite-time error bound given below.
\begin{theorem}\label{thm:final-theorem}
For any $k\geq0$, we have the following error bound:
\begin{align}\label{eqn:final-theorem}
\mathbb{E}[ \lVert Q_{k}^{A}-Q^{*} \rVert_{\infty}] &\leq 
\frac{120\alpha^{1/2}\lvert \mathcal{S}\times \mathcal{A}\rvert}{d_{\mathrm{min}}^{9/2}(1-\gamma)^{11/2}}+\frac{48\rho^{k-4}k^{4}\lvert \mathcal{S}\times \mathcal{A}\rvert^{3/2}}{(1-\gamma)}.
\end{align}
The same bound holds for $Q_{k}^{B}-Q^{*}$.
\end{theorem}
\noindent The proof is given in~\ref{sec:final-theorem}. 
The bound in~\eqref{eqn:final-theorem} consists of two terms with distinct
interpretations.
The second term decays exponentially fast as $k$ increases, since
$\rho\in(0,1)$, and therefore vanishes exponentially.
The first term represents a constant error term that depends on the
step size $\alpha$ and the minimum state--action visitation probability
$d_{\mathrm{min}}$.
By choosing a sufficiently small step size, this term can be made
arbitrarily small.
Moreover, $d_{\mathrm{min}}$ characterizes the level of exploration in the
learning process: under uniform exploration, $d_{\mathrm{min}}$ is large,
and it leads to a smaller error bound, whereas non-uniform or poor exploration
results in a smaller $d_{\mathrm{min}}$ and thus a larger error bound. The bound in~(\ref{eqn:final-theorem}) can be converted to more interpretable form presented below.
\begin{corollary}\label{cor:final-theorem}
For any $k\geq0$, we have the following error bound:
\begin{align}\label{eqn:final-theorem-cor}
\mathbb{E}[ \lVert Q_{k}^{A}-Q^{*} \rVert_{\infty}] &\leq 
\frac{120\alpha^{1/2}\lvert \mathcal{S}\times \mathcal{A}\rvert}{d_{\mathrm{min}}^{9/2}(1-\gamma)^{11/2}}+\frac{48\lvert \mathcal{S}\times \mathcal{A}\rvert^{3/2}}{(1-\gamma)}\frac{\rho^{-4}(-8)^{4}}{(\ln(\rho))^{4}}\rho^{\frac{-4}{\ln(\rho)}}\rho^{k/2}.
\end{align}
The same bound holds for $Q_{k}^{B}-Q^{*}$. 
\end{corollary}
\noindent The proof is given in~\ref{sec:final-corollaly}.
\subsubsection{Comparative convergence analysis}
We summarize in Table~\ref{table1} the sample complexities of representative double $Q$-learning and $Q$-learning algorithms, 
each derived under distinct assumptions and observation models. 
The comparison is organized along three key dimensions: 
(\textit{i}) the \textbf{sampling model}, which distinguishes between i.i.d.\ and non-i.i.d.\ data generation; 
(\textit{ii}) the \textbf{coverage condition}, which characterizes how sufficiently all state–action pairs are explored; and 
(\textit{iii}) the \textbf{step-size rule}, which determines whether the learning rate is constant or diminishing over time. 
Specifically, the coverage condition takes one of three forms: 
the \textit{cover-time} assumption, which requires that every state–action pair be visited at least once within a finite time window; 
the \textit{infinite-time covering} assumption, which ensures that every state--action pair is visited infinitely often over time; 
and assumes sampling from a stationary distribution with \textit{stochastic coverage},
meaning that each state–action pair has a strictly positive sampling probability. For completeness, asymptotic convergence results are also included to provide a broader perspective on the overall convergence landscape.

\vspace{2mm}
\noindent\textbf{Double \boldmath{$Q$}-learning.}
Earlier works such as~\cite{one,two} analyze the convergence properties of double $Q$-learning 
in a non-i.i.d. setting under cover-time assumptions. 
Specifically,~\cite{one} employs a polynomially decaying step-size. In contrast,~\cite{two} adopts a constant step-size. Our finite-time framework, by comparison, accommodates a general step-size $\alpha \in (0,1)$. Moreover, we instead assume an i.i.d. sampling with stochastic coverage.

\vspace{2mm}
\noindent\textbf{\boldmath{$Q$}-learning.}
For standard $Q$-learning,~\cite{chen2024lyapunov,li2020sample} focus on non-i.i.d.\ observation models with constant step-sizes, 
while~\cite{lim2024finite} adopts a similar non-i.i.d.\ setting but assumes a diminishing step-size rule. 
All three studies rely on the stochastic coverage assumption, ensuring that each state–action pair has a positive sampling probability. 
In contrast,~\cite{beck2012error,qu2020finite,even2003learning} also analyze non-i.i.d.\ sampling under cover-time coverage assumptions, 
where~\cite{beck2012error} employs a constant step-size, whereas~\cite{qu2020finite} and~\cite{even2003learning} adopt diminishing step-sizes. But~\cite{lee2024final} conducts an i.i.d. analysis that shares a similar system-theoretic foundation with our approach. 
Since these studies rely on different assumptions---such as constant step size versus diminishing step size, and cover-time or infinite-time covering versus stochastic coverage, and i.i.d. versus non-i.i.d. sampling---a 
direct numerical comparison among all methods is generally impractical. 
When compared to~\cite{lee2024final}, which follows a comparable i.i.d. and control-oriented analysis, 
our SDQ exhibits the same finite-time error bound order, scaling as 
$\mathcal{O}(|\mathcal{S}\!\times\!\mathcal{A}|^{3/2})$. 
The corresponding sample complexity, summarized in Table~\ref{table1} with 
$\tilde{\mathcal{O}}(\cdot)$ notation, 
represents the number of samples required to achieve an $\varepsilon$-accurate estimate of $Q^*$, 
derived from this finite-time bound.
This dependence arises from the cross-coupled structure of two interacting estimators, 
which introduces additional stochastic terms and higher-order dependence on $d_{\min}$ and $(1-\gamma)$.

\vspace{2mm}
\noindent\textbf{Discussion.}
Overall, the presented methods should be viewed as complementary rather than competing approaches. Each analysis is conducted under distinct assumptions and observation models, 
and thus emphasizes different aspects of the convergence behavior of $Q$-learning and double $Q$-learning. 
Our SDQ analysis does not aim to outperform existing analysis of $Q$-learning or double $Q$-learning in a theoretical sense, 
but rather to provide a unified interpretation based on a switching-system viewpoint 
and to establish finite-time expected error bounds within that framework. 
It should be noted that, under identical initialization ($Q_0^A = Q_0^B$), 
the SDQ update exactly reduces to the standard $Q$-learning algorithm, 
as discussed in Section~\ref{sec:MDQL_alg}. 
Hence, no theoretical improvement over $Q$-learning can be expected in this case. 
The contribution of this work lies not in achieving a tighter asymptotic rate, 
but in offering a generalized and control-theoretically interpretable framework 
that unifies $Q$-learning and double $Q$-learning within the same dynamical system formulation. 
Empirically, as presented in Section~\ref{sec:experiment}, 
the simultaneous update structure of SDQ tends to yield faster stabilization under random initialization, 
which supports the practical relevance and theoretical motivation of this study.

Furthermore, \Cref{thm:final-theorem} primarily focuses on finite-time estimation accuracy rather than bias analysis, the bias-reduction effect of SDQ arises implicitly from its cross-evaluation structure, each estimator uses the other’s greedy action as the target, thereby reducing the correlation between target selection and estimation noise, a mechanism analogous to that of standard double $Q$-learning.
Therefore, the proposed analysis should be regarded as a complementary and explanatory framework 
rather than a competing algorithmic enhancement.
The remaining parts of the paper are devoted to brief sketches of the proofs.

\begin{table}[h!]
\caption{Comparative analysis of several results: $t_{\rm mix}$ is the mixing time; $t_{\rm cover}$ is the cover time; $w\in(0,1)$ is a constant; $\tilde {\cal O}$ ignores polylogarithmic factors.}
\begin{center}
\renewcommand{\arraystretch}{1.4}
{\small
\begin{tabular}{c c c c}
\hline
\textbf{\small Method} & 
\textbf{\small Sample complexity} & 
\textbf{\small Step-size} & 
\textbf{\small Sampling / Coverage}\\
\hline
\multicolumn{4}{c}{\textbf{\small Simultaneous double $Q$-learning}}\\
\hline
Ours & 
$
\tilde{\mathcal{O}}\!\left(
\frac{|\mathcal{S}\times\mathcal{A}|^{2}}
{\varepsilon^{2}\,d_{\min}^{10}\,(1-\gamma)^{12}}
\right)$ 
& constant & i.i.d., stochastic \\

\hline
\multicolumn{4}{c}{\textbf{\small Double $Q$-learning}}\\
\hline

Lin et~al.~\cite{two} & 
$\tilde{\mathcal{O}}\!\left(
\frac{t_{\mathrm{cover}}}{(1-\gamma)^{7}\epsilon^{2}}
\right)$ 
& constant & non-i.i.d., cover-time \\

Xiong et~al.~\cite{one}  & 
$\tilde{\mathcal{O}}\!\left(
\left(\frac{t_{\mathrm{cover}}^4}{(1-\gamma)^6 \epsilon^2}\right)^{\!\frac{1}{\omega}}
+\left(\frac{t_{\mathrm{cover}}^2}{1-\gamma}\right)^{\!\frac{1}{1-\omega}}
\right)$ 
& diminishing & non-i.i.d., cover-time \\

Hasselt~\cite{hasselt2010double} &
-- (asymptotic convergence only) &
diminishing &
i.i.d., infinite-time covering \\

Weng et~al.~\cite{weng2020mean} &
-- (asymptotic convergence only) &
diminishing &
i.i.d., infinite-time covering \\

\hline
\multicolumn{4}{c}{\textbf{\small $Q$-learning}}\\
\hline

Lee et~al.~\cite{lee2024final} & 
$\tilde {\cal O}\!\left( {\frac{{\gamma ^2 d_{\max }^2 |{\cal S} \times {\cal A}|^2 }}{{\varepsilon ^2 d_{\min }^4 (1 - \gamma )^6 }}} \right)$ 
& constant & i.i.d., stochastic \\

Chen et~al.~\cite{chen2024lyapunov} & 
$\tilde {\cal O}\!\left(\frac{1}{d_{\min}^3(1-\gamma)^5\varepsilon^2}\right)$ 
& constant & non-i.i.d., stochastic \\

Li et~al.~\cite{li2020sample} & 
$\tilde {\cal O}\!\left( \frac{1}{d_{\min}(1-\gamma)^5 \varepsilon^2} + \frac{t_{\rm mix}}{d_{\min}(1-\gamma)}\right)$ 
& constant & non-i.i.d., stochastic \\

Lim et~al.~\cite{lim2024finite} & 
$\tilde{\mathcal{O}}\left(\frac{|\mathcal{S}\times\mathcal{A}|^{13} }{(1-\gamma)^{16}d_{\min}^{12}}\frac{1}{\epsilon^2} \right) $

& diminishing & non-i.i.d., stochastic \\

Beck et. al. \cite{beck2012error} & $\tilde {\cal O}\left( {\frac{t_{\rm cover}^3 |{\cal S}\times {\cal A}|}{(1-\gamma)^5 \varepsilon^2}} \right)$ & constant &  non-i.i.d., cover-time  \\

Qu et. al. \cite{qu2020finite} & $\tilde {\cal O}\left( {\frac{{{t_{{\rm{mix}}}}|{\cal S} \times {\cal A}{|^2}}}{{{{(1 - \gamma )}^5}{\varepsilon ^2}}}} \right)$  & diminishing & non-i.i.d., cover-time \\

Even-Dar et. al. \cite{even2003learning} & $\tilde {\cal O}\left( {\frac{{{{({t_{{\rm{cover}}}})}^{\frac{1}{{1 - \gamma }}}}}}{{{{(1 - \gamma )}^4}{\varepsilon ^2}}}} \right)$  & diminishing & non-i.i.d., cover-time\\

Tsitsiklis~\cite{tsitsiklis1994asynchronous} & -- (asymptotic convergence only)
& diminishing & non-i.i.d., infinite-time covering \\

Jaakkola~\cite{jaakkola1994convergence} & -- (asymptotic convergence only)
& diminishing & non-i.i.d., infinite-time covering \\

Borkar et~ al.~\cite{borkar2000ode} & -- (asymptotic convergence only)
& diminishing & non-i.i.d., synchronous update  \\

\hline
\end{tabular}}
\end{center}

\vspace{1mm}
\raggedright

\noindent
\parbox{0.95\linewidth}{
\footnotesize
\textbf{Notes.}
$t_{\mathrm{mix}}$: time required for a Markov chain to approach its stationary distribution (mixing time);
$t_{\mathrm{cover}}$: minimum time needed for all state–action pairs to be visited at least once (cover time);
\textit{stochastic coverage}: sampling from a stationary distribution where each $(s,a)$ has strictly positive probability;
\textit{infinite-time covering}: every $(s,a)$ pair is visited infinitely often;
\textit{synchronous update}: all state–action pairs are updated simultaneously at each iteration, thus no exploration assumption is required.
$\tilde{\mathcal{O}}(\cdot)$ notation hides polylogarithmic factors and, in some cases, implicit dependence on $|\mathcal{S}\times\mathcal{A}|$ when not explicitly stated.
}

\label{table1}
\end{table}

\section{Framework for convergence analysis of SDQ}\label{sec:framework}
Before presenting the technical details, we briefly outline the main structure of the finite-time analysis.
The central challenge in analyzing SDQ stems from the coupled and switching nature of the two estimators, which introduces additional affine terms and stochastic disturbances compared to standard $Q$-learning.
To address this challenge, we first model SDQ as a discrete-time switching system.
We then construct two auxiliary comparison systems—an upper comparison system and a lower comparison system—that respectively bound the original dynamics from above and below.
The analysis proceeds by first controlling the evolution of the estimator disagreement $Q_k^A - Q_k^B$ through a dedicated error system.
Once this disagreement is shown to contract over time, the lower comparison system effectively reduces to a stable linear stochastic system, enabling finite-time error bounds to be derived.
Finally, combining the bounds from the comparison systems yields the finite-time expected error guarantees for SDQ. A detailed realization of this analysis plan, including the specific comparison systems and error dynamics, is provided in Section~\ref{sec:overall_plans}.
\subsection{Switching system model}
In this subsection, we introduce a switching system model of SDQ in (\ref{eqn:modified_update}). First of all, using the notation introduced in Section~\ref{sec:ass_def}, the modified update in~(\ref{eqn:modified_update}) can be compactly written as
\begin{align}\label{double_q_ode}
Q_{k+1}^{A}=Q_{k}^{A}+\alpha_{k}(DR+\gamma DP\Pi_{Q_{k}^{B}} Q_{k}^{A}-DQ_{k}^{A}+ w_{k}^{A}), \nonumber \\
Q_{k+1}^{B}=Q_{k}^{B}+\alpha_{k}(DR+\gamma DP\Pi_{Q_{k}^{A}}Q_{k}^{B}-DQ_{k}^{B}+w_{k}^{B}),
\end{align}
where
\begin{align}\label{eqn:noise}
w_{k}^{A}&=(e_{a_{k}}\otimes e_{s_{k}})r_{k}\nonumber+\gamma(e_{a_{k}}\otimes e_{s_{k}})(e_{s_{k}'})^{T}\Pi_{Q_{k}^{B}} Q_{k}^{A}\nonumber-(e_{a_{k}}\otimes e_{s_{k}})(e_{a_{k}}\otimes e_{s_{k}})^{T}Q_{k}^{A}\nonumber\\&\quad-(DR+\gamma DP\Pi_{Q_{k}^{B}}Q_{k}^{A}-DQ_{k}^{A}),\nonumber\\
w_{k}^{B}&=(e_{a_{k}}\otimes e_{s_{k}})r_{k}\nonumber+\gamma(e_{a_{k}}\otimes e_{s_{k}})(e_{s_{k}'})^{T}\Pi_{Q_{k}^{A}} Q_{k}^{B}\nonumber-(e_{a_{k}}\otimes e_{s_{k}})(e_{a_{k}}\otimes e_{s_{k}})^{T}Q_{k}^{B}\nonumber\\&\quad-(DR+\gamma DP\Pi_{Q_{k}^{A}}Q_{k}^{B}-DQ_{k}^{B}).
\end{align}
Here, $s_k\in\mathcal{S}$ and $a_k\in\mathcal{A}$ denote the state and action
visited at iteration $k$, respectively, and $s_k'\in\mathcal{S}$ denotes
the subsequent state generated according to the transition probability
$P(\cdot \mid s_k,a_k)$.
Next, using the optimal Bellman equation $(\gamma DP\Pi_{Q^{*}}-D)Q^{*}+DR=0$ with ($\ref{double_q_ode}$), one can obtain 
\begin{align}\label{eqn:original}
Q_{k+1}^{A}-Q^{*}&=(I-\alpha D)(Q_{k}^{A}-Q^{*})\nonumber+\alpha D\{\gamma P\Pi_{Q_{k}^{B}}Q_{k}^{A}-\gamma P\Pi_{Q^{*}}Q^{*}\}\nonumber+\alpha w_{k}^{A},\nonumber\\Q_{k+1}^{B}-Q^{*}&=(I-\alpha D)(Q_{k}^{B}-Q^{*})+\alpha D\{\gamma P\Pi_{Q_{k}^{A}}Q_{k}^{B}-\gamma P\Pi_{Q^{*}}Q^{*}\}+\alpha w_{k}^{B},
\end{align}
which is a linear switching system with an extra affine terms, $\alpha D\{\gamma P\Pi_{Q_{k}^{B}}Q_{k}^{A}-\gamma P\Pi_{Q^{*}}Q^{*}\}$ and $\alpha D\{\gamma P\Pi_{Q_{k}^{A}}Q_{k}^{B}-\gamma P\Pi_{Q^{*}}Q^{*}\}$, and the stochastic noises, $w_{k}^{A}$ and $ w_{k}^{B}$ ~\citep{lee2023discrete}. The main difficulty in analysing the system arises from the extra affine term and the stochastic noise. Without these terms, finite-time analysis would be straightforward since the stability of the system matrix could be directly analyzed. However, with the affine term, the analysis becomes more challenging. To address this difficulty, we introduce the lower and upper comparison systems as in~\citep{lee2023discrete}, which enable easier analysis. 

\subsection{Upper comparison system }
Let us first consider the $\textit{upper comparison system}$
\begin{align}\label{eqn:ups}
Q_{k+1}^{A_{U}}-Q^{*}&=(I+\alpha\gamma DP\Pi_{Q_{k}^{B}}-\alpha D)(Q_{k}^{A_{U}}-Q^{*})\nonumber+\alpha w_{k}^{A},\quad Q_{0}^{A_{U}}-Q^{*} \in \mathbb{R}^{\lvert \mathcal{S} \rvert \lvert \mathcal{A}\rvert},\nonumber\\
Q_{k+1}^{B_{U}}-Q^{*}&=(I+\alpha\gamma DP\Pi_{Q_{k}^{A}}-\alpha D)(Q_{k}^{B_{U}}-Q^{*})+\alpha w_{k}^{B},\quad Q_{0}^{B_{U}}-Q^{*} \in \mathbb{R}^{\lvert \mathcal{S} \rvert \lvert \mathcal{A}\rvert}.
\end{align}
Here, $Q_k^{A_U}$ and $Q_k^{B_U}$ denote the state--action value iterates of the
upper comparison system associated with $Q_k^{A}$ and $Q_k^{B}$, respectively.
The above systems are switching systems, which have system matrices $I+\alpha\gamma DP\Pi_{Q_{k}^{A}}-\alpha D$ and $I+\alpha\gamma DP\Pi_{Q_{k}^{B}}-\alpha D$. These matrices switch according to the changes of $Q_{k}^{A}$ and $Q_{k}^{B}$. 
We can prove that the trajectory of the upper comparison system bounds that of the original system from above.
\begin{proposition}\label{prop:ucs}
Suppose that $Q_{0}^{A_U}-Q^{*}\geq Q_{0}^A-Q^{*}$ and $Q_{0}^{B_U}-Q^{*}\geq Q_{0}^B-Q^{*}$ hold, where $\geq$ is the element-wise inequality. Then, we have
\begin{align*}
Q_{k}^{A_{U}}-Q^{*}\geq  Q_{k}-Q^{*},\quad Q_{k}^{B_{U}}-Q^{*}\geq Q_{k}-Q^{*}.
\end{align*}
for all $k\geq 0$. 
\end{proposition}
\noindent The proof is given in~\ref{sec:prop-ucs}.
\subsection{Lower comparison system}
Let us consider the $\textit{lower comparison system}$
\begin{align}\label{eq:lower-comaprison-system}
Q_{k+1}^{A_{L}}-Q^{*}&=(I+\alpha\gamma DP\Pi_{Q^{*}}-\alpha D)(Q_{k}^{A_{L}}-Q^{*})\nonumber+\alpha\gamma DP(\Pi_{Q_{k}^{B}}-\Pi_{Q^{*}})(Q_{k}^{A}-Q_{k}^{B})\nonumber+\alpha w_{k}^{A},\quad Q_{0}^{A_{L}}-Q^{*} \in \mathbb{R}^{\lvert \mathcal{S} \rvert \lvert \mathcal{A}\rvert},\nonumber\\
Q_{k+1}^{B_{L}}-Q^{*}&=(I+\alpha\gamma DP\Pi_{Q^{*}}-\alpha D)(Q_{k}^{B_{L}}-Q^{*})+\alpha\gamma DP(\Pi_{Q_{k}^{*}}-\Pi_{Q_{k}^{A}})(Q_{k}^{A}-Q_{k}^{B})+\alpha w_{k}^{B},\quad Q_{0}^{B_{L}}-Q^{*} \in \mathbb{R}^{\lvert \mathcal{S} \rvert \lvert \mathcal{A}\rvert},
\end{align}
Here, $Q_k^{A_L}$ and $Q_k^{B_L}$ denote the state--action value iterates of the
lower comparison system associated with $Q_k^{A}$ and $Q_k^{B}$, respectively. The stochastic noises $w_{k}^{A}$ and $w_{k}^{B}$ are identical to the original system ($\ref{eqn:original}$). As before, we can prove that the trajectory of the lower comparison system bounds that of the original system from below.
\begin{proposition}\label{prop:lcs}
Suppose that $Q_{0}^{A_{L}}-Q^{*}\leq Q_{0}^{A}-Q^{*}$ and $Q_{0}^{B_{L}}-Q^{*}\leq Q_{0}^{B}-Q^{*}$ hold, where $\leq$ is the element-wise inequality. Then, we have
\begin{align*}
Q_{k}^{A_{L}}-Q^{*}\leq  Q_{k}^{A}-Q^{*},\quad
Q_{k}^{B_{L}}-Q^{*}\leq  Q_{k}^{B}-Q^{*},
\end{align*}
for all $k\geq 0$. 
\end{proposition}
\noindent The proof is given in~\ref{sec:prop-lcs}. The lower comparison system~(\ref{eq:lower-comaprison-system}) can be seen as a linear system with the states $ Q_{k}^{A_L}-Q^{*}$ and $Q_{k}^{B_L}-Q^{*}$ and the system matrix $I+\alpha\gamma DP\Pi_{Q^*}-\alpha D$. Moreover, it also includes the extra terms, $\alpha\gamma DP(\Pi_{Q_{k}^{B}}-\Pi_{Q^{*}})(Q_{k}^{A}-Q_{k}^{B})$ and $\alpha\gamma DP(\Pi_{Q_{k}^{*}}-\Pi_{Q_{k}^{A}})(Q_{k}^{A}-Q_{k}^{B})$, which can be seen as external disturbances. To derive a finite-time error bound, one needs to establish bounds on the error $Q_{k}^{A}-Q_{k}^{B}$ first. Therefore, in the next subsections, we introduce an error system. 
\begin{figure*}[ht]
\centering
  \includegraphics[scale=0.32]{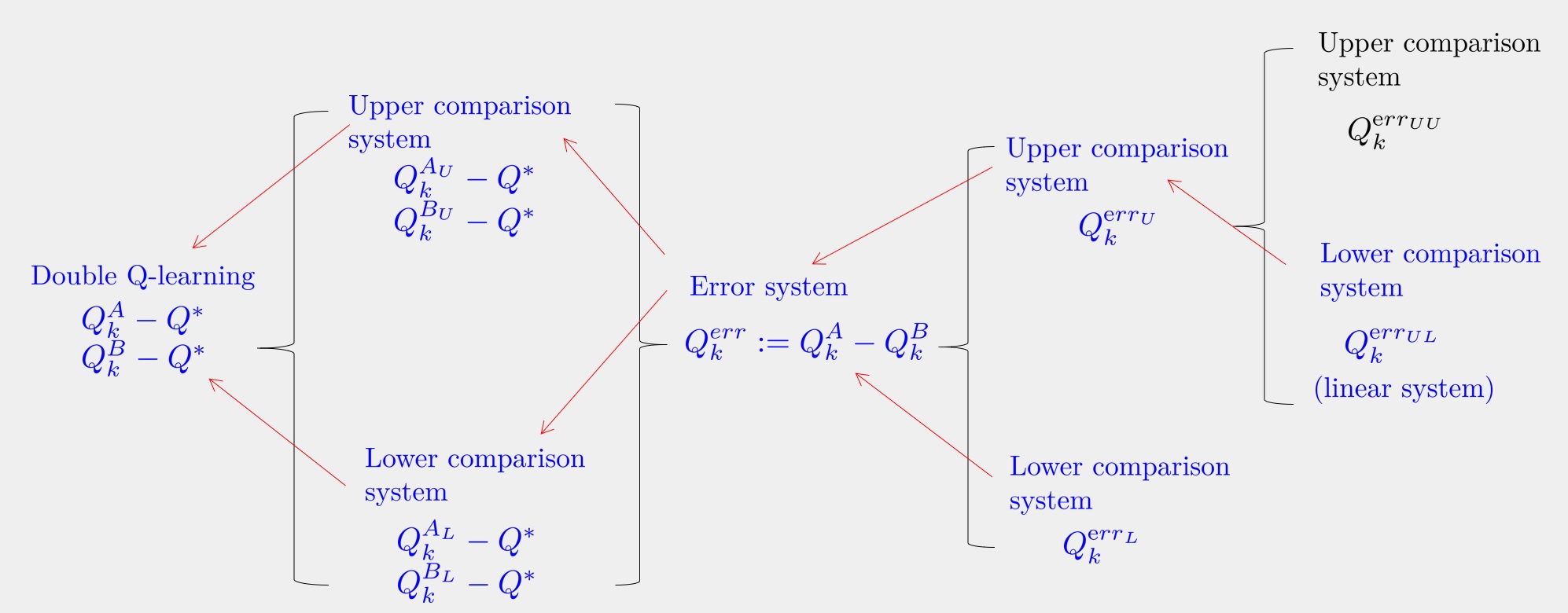}

  \caption{Overall flow of the proposed analysis} 
  \label{fig:diagram}
\end{figure*}

\subsection{Error system}
Let us consider the \textit{error system} with the state $Q_{k}^{\text{err}}\coloneqq Q_{k}^{A}-Q_{k}^{B}$ 
\begin{align}\label{eqn:error-system}
Q_{k+1}^{\text{err}}&=(I-\alpha D)Q_{k}^{\text{err}}+\alpha\gamma DP\Pi_{Q_{k}^{B}}Q_{k}^{A}-\alpha\gamma DP\Pi_{Q_{k}^{A}}Q_{k}^{B}+\alpha w_{k}^{A}-\alpha w_{k}^{B},\quad Q_{0}^{\text{err}}\in \mathbb{R}^{\lvert \mathcal{S} \rvert \lvert \mathcal{A}\rvert},
\end{align}
which can be obtained by subtracting the switching system model of $Q_{k}^{B}$ from that of $Q_{k}^{A}$ in (\ref{double_q_ode}).
The error system~(\ref{eqn:error-system}) can be seen as a linear system with the states $ Q_{k}^{\text{err}}$ and the system matrix $I-\alpha D$. Moreover, it includes extra affine term $\alpha\gamma DP\Pi_{Q_{k}^{B}}Q_{k}^{A}-\alpha\gamma DP\Pi_{Q_{k}^{A}}Q_{k}^{B}$. 

In the lower comparison system~(\ref{eq:lower-comaprison-system}), the extra terms, $\alpha\gamma DP(\Pi_{Q_{k}^{B}}-\Pi_{Q^{*}})Q_{k}^{\text{err}}$ and $\alpha\gamma DP(\Pi_{Q_{k}^{*}}-\Pi_{Q_{k}^{A}})Q_{k}^{\text{err}}$, make it hard to analyze the finite-time error bounds of the lower comparison system compared to the original $Q$-learning~\citep{lee2020unified,lee2023discrete}, where the lower comparison system is a linear system without the disturbance terms. To circumvent this difficulty, we will first prove that the error system $Q_{k}^{\text{err}}$ in the disturbance parts vanishes as $k \to \infty$. Intuitively, this implies that as the disturbance vanishes, and the lower comparison system converges to a pure stochastic linear system.

However, the error system in (\ref{eqn:error-system}) has the affine term $\alpha\gamma DP\Pi_{Q_{k}^{B}}Q_{k}^{A}-\alpha\gamma DP\Pi_{Q_{k}^{A}}Q_{k}^{B}$ similar to the original double $Q$-learning or $Q$-learning. Therefore, one can imagine that its convergence can be proved using similar techniques as in the $Q$-learning analysis~\citep{lee2020unified,lee2023discrete,lee2024final}. 
In particular, one can derive the upper and lower comparison systems of the error system, 
where these two auxiliary systems respectively provide upper and lower bounds on the evolution of the error trajectory. 
This allows the overall convergence to be established by showing that the true error remains confined between the two comparison systems, 
and the lower comparison system is linear. 
However, for the error system~(\ref{eqn:error-system}), similar ideas cannot be applied 
because the lower comparison system of the error system is a switching system. For this reason, we will use a different approach.

To this end, we will first consider an upper comparison system of the error system, 
and then derive a lower comparison system of the upper comparison system, which is linear. Let $Q_k^{\text{err}_U}$ denote the state of the upper comparison system, 
and let $A_{Q_k^{\text{err}_U}}$ be its corresponding system matrix that depends on $Q_k^{\text{err}_U}$. Conceptually, one could analyze the convergence of upper comparison system by adapting a standard autocorrelation-based method, 
which tracks the evolution of the second moment 
$\mathbb{E}[Q_k^{\mathrm{err}_U} (Q_k^{\mathrm{err}_U})^T]$ using a linear recursion as shown in Lemma~3 of Appendix. 
However the present upper comparison system forms a switching system
whose system matrix $A_{Q_k^{\mathrm{err}_U}}$ switches according to $Q_k^{\mathrm{err}_U}$ 
and depends probabilistically on its own state. Because of this dependence, taking expectation on both sides does not decouple the matrix and the state, 
and hence the simple linear recursion used in Lemma~3 
cannot be directly applied to obtain $\mathbb{E}[Q_k^{\mathrm{err}_U} (Q_k^{\mathrm{err}_U})^T]$. 
To handle this coupling effect rigorously, 
we retain the switching-system-based analysis for $Q_k^{\mathrm{err}_U}$, 
which provides a mathematically clear characterization of its convergence behavior. Then, we will consider a lower comparison system of the error system. 
Because the lower comparison system is also a switching system, 
we will derive a subtraction system, which can be obtained 
by subtracting the error lower comparison system from the error upper comparison system.

\section{Analysis process for convergence of SDQ}\label{sec:analysis_process}
\subsection{Overall plans}\label{sec:overall_plans}
The overall flow of the proposed analysis is given in Figure~\ref{fig:diagram}. 
The texts highlighted with blue indicate the dynamic systems we will deal with for our analysis. 
The red arrows represent the directions we will follow for the proof. The overall process is summarized as follows: \textbf{Step~1}: The finite-time error bound of $Q_{k}^{\text{err}_{UL}}$ is obtained by using its corresponding linear system structure. Then, based on the error bound on $Q_{k}^{\text{err}_{UL}}$, the finite-time error bound on $Q_{k}^{\text{err}_{U}}$ can be derived. \textbf{Step~2}: following similar lines as in Step~1, one can derive the error bound on $Q_{k}^{\text{err}}$ based on the error bound on $Q_{k}^{\text{err}_{U}}$ and $Q_{k}^{\text{err}_{U}}-Q_{k}^{\text{err}_{L}}$. \textbf{Step~3}: Using the error bound on $Q_{k}^{\text{err}}$ and the linear structures of $Q_{k}^{A_L} - Q^*$ and $Q_{k}^{B_L} - Q^*$, the finite-time error bounds on $Q_{k}^{A_L} - Q^*$ and $Q_{k}^{B_L} - Q^*$ can be derived. \textbf{Step~4}: By obtaining a subtraction system which can be obtained by subtracting the error lower comparison system from the error upper comparison system, the convergence of $Q_{k}^{A_U} - Q^*$ and $Q_{k}^{B_U} - Q^*$ can be shown.
\textbf{Step~5}: Using the previous results, we can  obtain a finite-time error bound on the iterates of SDQ. These steps will be detailed in Appendix.


\section{Conclusion}\label{sec:conclusion}
In this paper, we present a novel variant of double $Q$-learning, called SDQ, which mitigates the maximization bias of standard $Q$-learning by using two separate $Q$-estimators and eliminating the random selection step. By alternating the roles of the two estimators, SDQ offers a novel switching system interpretation. Empirical results indicate that SDQ converges faster than the original double $Q$-learning. Based on this representation, we derive new finite-time expected error bounds that complement existing results. Future work will focus on tightening the dimensional dependence of the theoretical bound by developing refined analytical techniques that account for the coupled structure of the two estimators. We also plan to extend SDQ to function approximation and adaptive settings to further enhance convergence and robustness.

\section*{Acknowledgement}
The work was supported by the Institute of Information Communications Technology Planning Evaluation (IITP) funded by the Korea government under Grant 2022-0-00469.

\newpage
\appendix
\renewcommand{\theassumption}{A.\arabic{assumption}}

\renewcommand{\thetheorem}{A.\arabic{theorem}}
\setcounter{theorem}{0}

\section{Convergence of stochastic linear system}\label{sec:auxiliary}
To aid in understanding the intricacies of the convergence of SDQ, as explained in \ref{sec:disturb} and \ref{remaining_part}, let us consider the following stochastic linear system, which offers a helpful conceptual framework:
\begin{align}\label{eqn:linear_system}
x_{k+1}=Ax_{k}+\alpha v_{k},\quad x_{0}\in\mathbb{R}^{n},\quad k\in\{0,1,...\},
\end{align}
where $x_{k}\in\mathbb{R}^{n}$ is the state, and $A$ is system matrix, and  $\alpha v_{k}$ is the stochastic noise with a constant $\alpha\in(0,1)$. Here, we will first investigate a finite-time error analysis of the state of~(\ref{eqn:linear_system}), and it will be used in the proof of SDQ.
To this end, let us assume that the noise energy of $v_{k}$ is bounded as $\mathbb{E}[v_{k}^{T}v_{k}]\leq V_{\text{max}}$ and $V_{\text{max}}>0$. Then, it can be proved that the maximum eigenvalue of $\mathbb{E}[v_{k}v_{k}^{T}]$ can be bounded by $V_{\text{max}}$.
\begin{lemma}[\citep{lee2024final}]\label{lem:maximum_eigen_value}
The maximum eigenvalue of $\mathbb{E}[v_{k}v_{k}^{T}]$ is bounded as 
\[
\lambda_{\mathrm{max}}(\mathbb{E}[v_{k}v_{k}^{T}])\leq V_{\mathrm{max}} 
\]
for all $k\geq0$, where $V_{\mathrm{max}}>0$ is from our assumption.
\end{lemma}
\begin{proof}
The proof is completed by noting $\lambda_{\text{max}}(\mathbb{E}[v_{k}v_{k}^{T}])\leq\text{tr}(\mathbb{E}[v_{k}v_{k}^{T}])=\mathbb{E}[\text{tr}(v_{k}v_{k}^{T})]=\mathbb{E}[v_{k}^{T}v_{k}]\leq V_{\text{max}},$ where the last inequality comes from our assumption and the second equality uses the fact that the trace is a linear function. This completes the proof.
\end{proof}
\noindent Moreover, let us assume that the system matrix $A$ is also bounded. 

\begin{assumption}\label{asm:appendix_one}
The system matrix $A$ satisfies $\|A\|_{\infty} \le \rho$ for some constant $\rho \in (0,1)$.
\end{assumption}

\noindent As a next step, we investigate how the auto-correlation matrix $\mathbb{E}[x_{k}x_{k}^{T}]$ propagates over the time. Thus, one can consider the auto-correlation matrix of the state recursively calculated as follows: 
\[
\mathbb{E}[x_{k+1}x_{k+1}^{T}]=A\mathbb{E}[x_{k}x_{k}^{T}]A^{T}+\alpha^{2}V_{k},
\]
where $\mathbb{E}[v_{k}v_{k}^{T}]=V_{k}$.
Defining $X_{k}\coloneqq \mathbb{E}[x_{k}x_{k}^{T}],k \geq 0$, the above recursion can be written by
\[
X_{k+1}=AX_{k}A^{T}+\alpha^{2}V_{k}, \quad \forall k\geq 0.
\]
To prove the convergence of~(\ref{eqn:linear_system}), we first establish a bound on the trace of $X_{k}$.

\begin{lemma}[\citep{lee2024final}]\label{app:bounded_trace}
We have the following bound:
\begin{align}\label{eqn:Theorem1}
\mathrm{tr}(X_{k})\leq \frac{9n^{2}\alpha}{d_{\mathrm{min}}(1-\gamma)^{3}}+\lVert x_{0}\rVert^{2}_{2}n^{2}\rho^{2k}\nonumber
\end{align}
\end{lemma}

\begin{proof}
We first bound $\lambda_{\text{max}}(X_{k})$ as follows:
\[
\begin{split}
\lambda_{\text{max}}(X_{k})&\leq \alpha^{2}\sum_{i=0}^{k-1}\lambda_{\text{max}}(A^{i}V_{k-i-1}(A^{T})^{i})+\lambda_{\text{max}}(A^{k}X_{0}(A^{T})^{k})\\&\leq\alpha^{2}\sup_{j\geq 0}\lambda_{\text{max}}(V_{j})\sum_{i=0}^{k-1}\lambda_{\text{max}}(A^{i}(A^{T})^{i})+\lambda_{\text{max}}(X_{0})\lambda_{\text{max}}(A^{k}(A^{T})^{k})\\
&=\alpha^{2}\sup_{j\geq 0}\lambda_{\text{max}}(V_{j})\sum_{i=0}^{k-1}\lVert A^{i}\rVert_{2}^{2}+\lambda_{\text{max}}(X_{0})\lVert A^{k}\rVert_{2}^{2}\\
&\leq \alpha^{2}V_{\text{max}}n\sum_{i=0}^{k-1}\lVert A^{i}\rVert_{\infty}^{2}+n\lambda_{\text{max}}(X_{0})\lVert A^{k}\rVert_{\infty}^{2}\\&\leq \alpha^{2}V_{\text{max}}n\sum_{i=0}^{k-1}\rho^{2i}+n\lambda_{\text{max}}(X_{0})\rho^{2k}\\&\leq \alpha^{2}V_{\text{max}}n\lim_{k\rightarrow\infty}\sum_{i=0}^{k-1}\rho^{2i}+n\lambda_{\text{max}}(X_{0})\rho^{2k}\\
&\leq\frac{\alpha^{2}V_{\text{max}}n}{1-\rho^{2}}+n\lambda_{\text{max}}(X_{0})\rho^{2k}\\&\leq\frac{\alpha^{2}V_{\text{max}}n}{1-\rho}+n\lambda_{\text{max}}(X_{0})\rho^{2k}
\end{split}
\]
where the first inequality is due to $A^{i}V_{k-i-1}(A^{T})^{i}\succeq0$ and $A^{k}X_{0}(A^{T})^{k}\succeq0$, the third inequality comes from \cref{lem:maximum_eigen_value}, 
$\lVert \cdot \rVert_{2}\leq \sqrt{n}\lVert \cdot \rVert_{\infty}$, the fourth inequality is due to \cref{asm:appendix_one}, and the sixth and last inequalities come from $\rho \in(0,1)$. On the other hand, since $X_{k}\succeq 0$, the diagonal elements are nonnegative. Therefore, we have $\text{tr}(X_{k})\leq n\lambda_{\text{max}}(X_{k})$. Combining the last two inequalities leads to 
\[
\text{tr}(X_{k})\leq n\lambda_{\text{max}}(X_{k})\leq\frac{\alpha^{2}V_{\text{max}}n^{2}}{1-\rho}+n^{2}\lambda_{\text{max}}(X_{0})\rho^{2k}
\]
Moreover, noting the inequality $\lambda_{\text{max}}(X_{0})\leq \text{tr}(X_{0}) = \text{tr}(x_{0}x_{0}^{T})=\lVert x_{0} \rVert_{2}^{2}$, and plugging $\rho=1-\alpha d_{\text{min}}(1-\gamma)$ into $\rho$ in the last inequality, one gets the desired conclusion.
\end{proof}
Now, we are ready to present a finite-time bound on the state $x_{k}$ of~(\ref{eqn:linear_system}). 
\begin{theorem}[\citep{lee2024final}]\label{app:stochasit-linear}
For any $k\geq 0$, we have
\begin{align}\label{eqn:Theorem1}
\mathbb{E} \left[\lVert x_{k}\rVert_{2} \right]&\leq \frac{3\alpha^{1/2}n}{d_{\mathrm{min}}^{1/2}(1-\gamma)^{3/2}}+n \lVert x_{0}\rVert_{2}\rho^{k}.
\end{align}
\end{theorem}
\begin{proof}
Noting the relations
\[
\begin{split}
\mathbb{E}[\lVert x_{k}\rVert_{2}^{2}]=\mathbb{E}[ x_{k}^{T}x_{k} ]=\mathbb{E}[ \text{tr}(x_{k}^{T}x_{k})]=\mathbb{E}[ \text{tr}(x_{k}x_{k}^{T})]=\mathbb{E}[ \text{tr}(X_{k})],
\end{split}
\]
and using the bound in \cref{app:bounded_trace}, one gets
\begin{align*}
\mathbb{E} \left[\lVert x_{k}\rVert_{2}^{2} \right]&\leq \frac{9\alpha n^{2}}{d_{\text{min}}(1-\gamma)^{3}}+n^{2}\lVert x_{0}\rVert_{2}^{2}\rho^{2k}
\end{align*}
Taking the square root on both side of the last inequality, using the subadditivity of the square root function, the Jensen inequality, and the concavity of the square root function, we have the desired conclusion.
\end{proof}

Note that the result in~\cref{app:stochasit-linear} will be used in our main analysis of SDQ. In particular, we will use the following form of the state of the system:
\begin{align}
x_{i}= A^{i}x_{0}+\sum_{j=0}^{i-1}\alpha A^{(i-1)-j}v_{j},\label{eq:1}
\end{align}
which can be obtained by summing the recursion in~(\ref{eqn:linear_system}) from $k=0$ to $k=i$. 
Based on the above expression, \cref{app:stochasit-linear} can be presented different form as follows.

\begin{corollary}\label{coro:finite_time_bound}
For any $k\geq 0$, we have
\[
\mathbb{E}\biggl[\biggl\lVert A^{k}x_{0}+\sum_{j=0}^{k-1}\alpha A^{(k-1)-j}v_{j}\biggr\rVert_{2}\biggl]\leq \frac{3\alpha^{1/2}n}{d_{\mathrm{min}}^{1/2}(1-\gamma)^{3/2}}+n \lVert x_{0}\rVert_{2} \rho^{k}.
\]
\end{corollary}

\begin{proof}
The proof can be done directly from~(\ref{eq:1}) and~\cref{app:stochasit-linear}.   
\end{proof}

\section{Detailed analysis result of the convergence of SDQ ($Q_{k}^{\text{err}}$ part)}\label{sec:disturb}

To establish a finite-time error bound in this paper, the main challenge is to establish a bound of the error system $Q_{k}^{\text{err}}$. The overall analysis strategy is presented in Section~\ref{sec:overall_plans} briefly. We derive the convergence of $Q_{k}^{\text{err}}$ using the following two steps: 
\begin{itemize}
    \item Step~1: A finite-time error bound of $Q_{k}^{\text{err}_{UL}}$ is obtained by using its corresponding linear system structure. Then, based on the error bound on $Q_{k}^{\text{err}_{UL}}$, a finite-time error bound on $Q_{k}^{\text{err}_{U}}$ can be derived. 
    \item Step~2: Next, following similar lines as in Step~1, one can derive an error bound on $Q_{k}^{\text{err}}$ based on the error bound on $Q_{k}^{\text{err}_{U}}$ and $Q_{k}^{\text{err}_{L}}$.
\end{itemize}

In this section, we present a detailed analysis process. To establish the groundwork for our proof, we first introduce an auxiliary lemma demonstrating the nonnegativity and boundedness of the system matrix.
Before proving the boundedness result, 
we first show that $A_Q$ is elementwise nonnegative, 
which will be used in the subsequent lemma.

\begin{lemma}
[\citep{lee2024final}]\label{lemma:nonnegative-matrix} For any $Q \in {\mathbb R}^{|{\cal S}\times {\cal A}|}$, $A_Q$ is a nonnegative matrix (all entries are nonnegative). \end{lemma} \begin{proof} Recalling the definition $A_Q :=I + \alpha(\gamma DP\Pi_Q-D)$, one can easily see that for any $i,j \in \{ 1,2, \ldots ,|{\cal S}\times {\cal A}|\} $, we have ${[{A_Q}]_{ij}} = {[I - \alpha D + \alpha \gamma DP{\Pi _Q}]_{ij}} = {[I - \alpha D]_{ij}} + \alpha \gamma {[DP{\Pi _Q}]_{ij}} \ge 0$, where ${[ \cdot ]_{ij}}$ denotes the element of a matrix $[ \cdot ]$ in the $i$th row and $j$th column, and the inequality follows from the fact that both $I - \alpha D$ and $DP{\Pi _Q}$ are nonnegative matrices. This completes the proof. \end{proof}

Having established that $A_Q$ is elementwise nonnegative, 
we next analyze its boundedness property, 
which plays a crucial role in ensuring the stability of the subsequent system dynamics.
\begin{lemma}[\citep{lee2024final}]
\label{app:system_upper_bound}
For any $Q\in {\mathbb R}^{\lvert \mathcal{S} \rvert\lvert \mathcal{A} \rvert}$, we have
\begin{align}
\lVert A_{Q} \rVert_{\infty} \leq \rho, \nonumber
\end{align}
where the matrix norm $\lVert A\rVert_{\infty}\coloneqq max_{1\leq i\leq m}\sum_{j=1}^{n}\lvert A_{ij}\rvert$ and $A_{ij}$ is the element of $A$ in $i$-th row and $j$-th column.
\end{lemma}
\begin{proof}
Note the following identities
\[
\begin{split}
\sum_{j}\lvert [A_{Q}]_{ij}\rvert &= \sum_{j}\lvert [I-\alpha D+\alpha\gamma DP\Pi_{Q}]_{ij}\rvert\\
&=[I-\alpha D]_{ii}+\sum_{j}[\alpha\gamma DP\Pi_{Q}]_{ij}\\
&=1-\alpha[D]_{ii}+\alpha\gamma[D]_{ii}\sum_{j}[P\Pi_{Q}]_{ij}\\
&=1-\alpha[D]_{ii}+\alpha\gamma[D]_{ii}\\
&=1+\alpha[D]_{ii}(\gamma-1),
\end{split}
\]
where the second line is due to the fact that $A_{Q}$ is a non-negative matrix. Taking the maximum over $i$, we have
\[
\begin{split}
\lVert A_{Q}\rVert_{\infty}&=\max_{i\in\{1,2,...,\lvert S\rvert\lvert A\rvert\}}{1+\alpha[D]_{ii}(\gamma-1)}\\
&=1-\alpha\min_{(s,a)\in \mathcal{S}\times\mathcal{A}}d(s,a)(1-\gamma)\\&=\rho,
\end{split}
\]
which completes the proof.
\end{proof}

As the first step, we present a convergence analysis of $Q_{k}^{\text{err}_{U}}$ in next subsection.

\subsection{Convergence of $Q_{k}^{\mathrm{err}_{U}}$}
Let us write the error upper comparison system $Q_{k}^{\text{err}_{U}}$ as follows:
\begin{align}\label{eqn:disturb_ucs}
Q_{k+1}^{\text{err}_{U}}=(I+\alpha\gamma DP\Pi_{Q_{k}^{\text{err}_{U}}}-\alpha D) Q_{k}^{\text{err}_{U}}+\alpha w_{k}^{A}-\alpha w_{k}^{B},\quad Q_{0}^{\text{err}_{U}} \in \mathbb{R}^{\lvert \mathcal{S} \rvert \lvert \mathcal{A}\rvert},
\end{align}
where the stochastic noises, $w_{k}^{A}$ and $w_{k}^{B}$, are identical to those of the original system in~(\ref{double_q_ode}). In the following proposition, we prove that $Q_{k}^{\text{err}_{U}}$ upper bounds $Q_{k}^{\text{err}}$. 
\begin{proposition}\label{prop:disturb_ucs}
Suppose $Q_{0}^{\mathrm{err}_{U}}\geq Q_{0}^{\mathrm{err}}$, where ``$\geq$'' is used as the element-wise inequality. Then, we have
\[
Q_{k}^{\mathrm{err}_{U}}\geq Q_{k}^{\mathrm{err}},
\]
for all $k\geq 0.$
\end{proposition}
\begin{proof}
The proof is completed by an induction argument. Suppose that $Q_{i}^{\text{err}_{U}}\geq Q_{i}^{\text{err}}$ holds for $0\le i \le k$. Then, it follows that
\[
\begin{split}
Q_{k+1}^{\text{err}} &=Q_{k}^{\text{err}}+\alpha\gamma DP\Pi_{Q_{k}^{B}}Q_{k}^{A}-\alpha\gamma DP\Pi_{Q_{k}^{A}} Q_{k}^{B}-\alpha DQ_{k}^{\text{err}}+\alpha w_{k}^{A}-\alpha w_{k}^{B} \\&\leq Q_{k}^{\text{err}}+\alpha\gamma DP\Pi_{Q_{k}^{A}}Q_{k}^{A}-\alpha\gamma DP\Pi_{Q_{k}^{A}} Q_{k}^{B}-\alpha DQ_{k}^{\text{err}}+\alpha w_{k}^{A}-\alpha w_{k}^{B}
\\
&= Q_{k}^{\text{err}}+\alpha\gamma DP\Pi_{Q_{k}^{A}}Q_{k}^{\text{err}}-\alpha DQ_{k}^{\text{err}}+\alpha w_{k}^{A}-\alpha w_{k}^{B}\\
&=(I+\alpha\gamma DP\Pi_{Q_{k}^{A}}-\alpha D)Q_{k}^{\text{err}}+\alpha w_{k}^{A}-\alpha w_{k}^{B}\\
&\leq(I+\alpha\gamma DP\Pi_{Q_{k}^{\text{err}}}-\alpha D)Q_{k}^{\text{err}}+\alpha w_{k}^{A}-\alpha w_{k}^{B}\\
&\leq(I+\alpha\gamma DP\Pi_{Q_{k}^{\text{err}}}-\alpha D)Q_{k}^{\text{err}_{U}}+\alpha w_{k}^{A}-\alpha w_{k}^{B}\\&\leq(I+\alpha\gamma DP\Pi_{Q_{k}^{\text{err}_{U}}}-\alpha D)Q_{k}^{\text{err}_{U}}+\alpha w_{k}^{A}-\alpha w_{k}^{B} 
\\&=Q_{k+1}^{\text{err}_{U}},
\end{split}
\]
where the first inequality is due to $\Pi_{Q_{k}^{A}}Q_{k}^{A}\geq \Pi_{Q_{k}^{B}}Q_{k}^{A}$ and the second inequality is due to 
$\Pi_{Q_{k}^{A}-Q_{k}^{B}}Q_{k}^{\text{err}}\geq \Pi_{Q_{k}^{A}}Q_{k}^{\text{err}}$, respectively, and the third inequality is due to the hypothesis $Q_{k}^{\text{err}_{U}}\geq Q_{k}^{\text{err}}$ and the fact that the matrix $I+\alpha\gamma DP\Pi_{Q_{k}^{\text{err}}}-\alpha D$ is nonnegative, i.e., all elements are nonnegative by \cref{lemma:nonnegative-matrix}. Therefore, $Q_{k+1}^{\text{err}_{U}}\geq Q_{k+1}^{\text{err}}$ holds, and the proof is completed by induction.
\end{proof}

\noindent To prove the convergence of $Q_{k}^{\mathrm{err}_{U}}$, we consider another comparison system $Q_{k}^{\mathrm{err}_{UL}}$ which is a lower comparison system of $Q_{k}^{\mathrm{err}_{U}}$. 
In the following subsection, a convergence analysis of $Q_{k}^{\mathrm{err}_{UL}}$ is presented.

\subsection{Convergence of $Q_{k}^{\mathrm{err}_{UL}}$}\label{sec:convergence-err-ul}
Let us write the $Q_{k}^{\mathrm{err}_{UL}}$, which has the following form:
\begin{align}\label{eqn:disturb_lcs}
Q_{k+1}^{\mathrm{err}_{UL}}=(I+\alpha\gamma DP\Pi_{Q^{*}}-\alpha D)Q_{k}^{\mathrm{err}_{UL}}+\alpha w_{k}^{A}-\alpha w_{k}^{B},
\end{align}
where the stochastic noises $w_{k}^{A}$ and $w_{k}^{B}$ are identical to those in the original system (\ref{double_q_ode}). Note that the system is the lower comparison system of the upper comparison system corresponding to $Q_{k}^{\mathrm{err}}$. 
In the following proposition, we prove that $Q_{k}^{\text{err}_{UL}}$ is a lower comparison system of $Q_{k}^{\text{err}_{U}}$. 

\begin{proposition}\label{prop:disturb_lcs}
Suppose $Q_{0}^{\text{err}_{U}}\geq Q_{0}^{\text{err}_{UL}}$, where ``$\geq$'' is used as the element-wise inequality. Then, we have
\[
Q_{k}^{\mathrm{err}_{U}}\geq Q_{k}^{\mathrm{err}_{UL}},
\]
for all $k\geq 0$.
\end{proposition}
\begin{proof}
The proof is completed by an induction argument. Suppose that $Q_{i}^{\text{err}_{U}}\geq Q_{i}^{\text{err}_{UL}}$ holds for $0\le i \le k$. Then, it follows from~({\ref{eqn:disturb_ucs}}) that
\[
\begin{split}
Q_{k+1}^{\text{err}_{U}}&=(I+\alpha\gamma DP\Pi_{Q_{k}^{\text{err}_{U}}}-\alpha D)Q_{k}^{\text{err}_{U}}+\alpha w_{k}^{A}-\alpha w_{k}^{B}\\&\geq (I+\alpha\gamma DP\Pi_{Q^{*}}-\alpha D)Q_{k}^{\text{err}_{U}}+\alpha w_{k}^{A}-\alpha w_{k}^{B}\\&\geq (I+\alpha\gamma DP\Pi_{Q^{*}}-\alpha D)Q_{k}^{\text{err}_{UL}}+\alpha w_{k}^{A}-\alpha w_{k}^{B}\\&=Q_{k+1}^{\text{err}_{UL}},
\end{split}
\]
where the first inequality is due to $\Pi_{Q_{k}^{\text{err}_{U}}}Q_{k}^{\text{err}_{U}}\geq \Pi_{Q^{*}}Q_{k}^{\text{err}_{U}}$ and the second inequality is due to the hypothesis 
$Q_{k}^{\text{err}_{U}}\geq Q_{k}^{\text{err}_{UL}} $ and the fact that the matrix $I+\alpha\gamma DP\Pi_{Q^{*}}-\alpha D$ is nonnegative, i.e., all elements are nonnegative by \cref{lemma:nonnegative-matrix}. Therefore, $Q_{k+1}^{\text{err}_{U}}\geq Q_{k+1}^{\text{err}_{UL}}$ holds, and the proof is completed by induction.

\end{proof}

The system (\ref{eqn:disturb_lcs}) is a stochastic linear system with system matrix $I+\alpha\gamma DP\Pi_{Q^{*}}-\alpha D$ and noise $w_{k}^{A}-w_{k}^{B}$. To establish the convergence bound of this system, the same analysis approach as in~\ref{sec:auxiliary} can be applied.
In particular, let us define $x_{k}\coloneqq Q_{k}^{\text{err}_{UL}}$ and $A\coloneqq I+\alpha\gamma DP\Pi_{Q^{*}}-\alpha D$. Then, the system (\ref{eqn:disturb_lcs}) can be presented as the following stochastic linear system:
\begin{align}
x_{k+1}=Ax_{k}+\alpha (w_k^{A}-w_k^{B}),\quad x_{0}\in\mathbb{R}^{n}, \quad \forall k \geq 0,\label{eq:2}
\end{align}
where the noise term $w_{k}^{A}-w_{k}^{B}$ can be written as
\[
w_{k}^{A}-w_{k}^{B}=(e_{a_{k}}\otimes e_{s_{k}})(\delta_{k}^{A}-\delta_{k}^{B})-\gamma 
DP(\Pi_{Q^{B}_{k}}Q^{A}_{k}-\Pi_{Q^{A}_{k}}Q^{B}_{k})+D(Q_{k}^{A}-Q_{k}^{B}), \\
\]
where
\[
\begin{split}
w_{k}^{A}&=(e_{a_{k}}\otimes e_{s_{k}})r_{k}^{A}+\gamma(e_{a_{k}}\otimes e_{s_{k}})(e_{s_{k}'})^{T}\Pi_{Q^{B}_{k}}Q^{A}_{k}-(e_{a_{k}}\otimes e_{s_{k}})(e_{a_{k}}\otimes e_{s_{k}})^{T}Q^{A}_{k}\\&\quad-(DR+\gamma DP\Pi_{Q^{B}_{k}}Q^{A}_{k}-DQ_{k}^{A}) \\
w_{k}^{B}&=(e_{a_{k}}\otimes e_{s_{k}})r_{k}^{B}+\gamma(e_{a_{k}}\otimes e_{s_{k}})(e_{s_{k}'})^{T}\Pi_{Q^{A}_{k}}Q^{B}_{k}-(e_{a_{k}}\otimes e_{s_{k}})(e_{a_{k}}\otimes e_{s_{k}})^{T}Q^{B}_{k}\\&\quad-(DR+\gamma DP\Pi_{Q^{A}}Q^{B}-DQ_{k}^{B}) \\
\delta_{k}^{A}&\coloneqq r_{k}^{A}+\gamma(e_{s_{k}'}^{T})\Pi_{Q_{k}^{B}}Q_{k}^{A}
-(e_{a_{k}}\otimes e_{s_{k}})^{T}Q_{k}^{A}\\
\delta_{k}^{B}&\coloneqq 
r_{k}^{B}+\gamma(e_{s_{k}'}^{T})\Pi_{Q_{k}^{A}}Q_{k}^{B}
-(e_{a_{k}}\otimes e_{s_{k}})^{T}Q_{k}^{B}
\end{split}
\]
Here, $r_{k+1}^{A}$ and $r_{k+1}^{B}$ denote the instantaneous rewards observed at
iteration $k$ for the updates of $Q_k^{A}$ and $Q_k^{B}$, respectively. 
To prove the convergence of $Q_{k}^{\text{err}_{UL}}$, we prove the boundedness of the noise term in~(\ref{eq:2}). 
The boundedness of $w_{k}^{A}-w_{k}^{B}$ in~(\ref{eq:2}) is formally proved in the following lemma.
\begin{lemma}\label{app:upper_bound_new_noise}
The noise term $w_{k}^{A}-w_{k}^{B}$ in~(\ref{eq:2}) satisfies
\[
\mathbb{E}\left[(w_{k}^{A}-w_{k}^{B})^{T}(w_{k}^{A}-w_{k}^{B})\right]
\leq \frac{16}{(1-\gamma)^{2}} \coloneqq W_{\mathrm{max}},
\]
for all $k\geq 0$
\end{lemma}
\begin{proof}
One can get the following bound on $\mathbb{E}\left[(w_{k}^{A}-w_{k}^{B})^{T}(w_{k}^{A}-w_{k}^{B})\right]$:

\begin{align*}
&\mathbb{E}\bigl[(w_{k}^{A}-w_{k}^{B})^{T}(w_{k}^{A}-w_{k}^{B})\bigr] \\
&= \mathbb{E}\Bigl[\bigl\lVert
   (e_{a_{k}}\otimes e_{s_{k}})(\delta_{k}^{A}-\delta_{k}^{B})
   -\bigl(\gamma DP(\Pi_{Q_{k}^{B}}Q_{k}^{A}-\Pi_{Q_{k}^{A}}Q_{k}^{B})
   -D(Q_{k}^{A}-Q_{k}^{B})\bigr)
   \bigr\rVert_{2}^{2}\Bigr] \\
&= \mathbb{E}\bigl[(\delta_{k}^{A}-\delta_{k}^{B})^{2}\bigr]
  -\Bigl\lVert
   \gamma DP(\Pi_{Q_{k}^{B}}Q_{k}^{A}-\Pi_{Q_{k}^{A}}Q_{k}^{B})
   -D(Q_{k}^{A}-Q_{k}^{B})
   \Bigr\rVert_{2}^{2} \\
&\le \mathbb{E}\bigl[(\delta_{k}^{A}-\delta_{k}^{B})^{2}\bigr] \\
&= \mathbb{E}\Bigl[\bigl(
     r_{k+1}^{A} + \gamma\,e_{s_{k}'}^{T}\Pi_{Q_{k}^{B}}Q_{k}^{A}
     -(e_{a_{k}}\otimes e_{s_{k}})^{T}Q_{k}^{A} 
     -\,\bigl(
         r_{k+1}^{B} + \gamma\,e_{s_{k}'}^{T}\Pi_{Q_{k}^{A}}Q_{k}^{B}
         -(e_{a_{k}}\otimes e_{s_{k}})^{T}Q_{k}^{B}
       \bigr)
   \bigr)^{2}\Bigr] \\
&\le \mathbb{E}\Bigl[\bigl(
     \lvert r_{k+1}^{A}\rvert
     + \lvert\gamma\,e_{s_{k}'}^{T}\Pi_{Q_{k}^{B}}Q_{k}^{A}\rvert
     + \lvert(e_{a_{k}}\otimes e_{s_{k}})^{T}Q_{k}^{A}\rvert 
     +\,\lvert r_{k+1}^{B}\rvert
     + \lvert\gamma\,e_{s_{k}'}^{T}\Pi_{Q_{k}^{A}}Q_{k}^{B}\rvert
     + \lvert(e_{a_{k}}\otimes e_{s_{k}})^{T}Q_{k}^{B}\rvert
   \bigr)^{2}\Bigr] \\
&= \frac{16}{(1-\gamma)^{2}}
  = W_{\max}
\end{align*}
\end{proof}
\noindent This bound on the noise term plays a key role in establishing the
finite-time convergence of the stochastic linear system~\eqref{eq:2}.
Now, because~(\ref{eqn:disturb_lcs}) is a stochastic linear system, the analysis of a simple stochastic linear system from~\ref{sec:auxiliary} can be applied directly. Then, we can get the upper bound of $Q_{k}^{\text{err}_{UL}}$ in the following lemma.

\begin{lemma}\label{lem:err_ul_upper_bound}
For any $k\geq 0$, we have
\[
\mathbb{E}[\lVert Q_{k}^{\mathrm{err}_{UL}}\rVert_{2}]\leq \frac{4\alpha^{1/2}\lvert \mathcal{S}\times\mathcal{A}\rvert}{d_{\mathrm{min}}^{1/2}(1-\gamma)^{3/2}}+\lvert \mathcal{S}\times \mathcal{A}\rvert \lVert Q_{0}^{\mathrm{err}_{UL}}\rVert_{2}\rho^{k}.
\]
\end{lemma}
\begin{proof}
Noting the relations
\[
\begin{split}
\mathbb{E}[\lVert Q_{k}^{\text{err}_{UL}}\rVert_{2}^{2}]&=\mathbb{E}[(    Q_{k}^{\text{err}_{UL}}  )^{T}(Q_{k}^{\text{err}_{UL}})]\\&=\mathbb{E}[ \text{tr}(Q_{k}^{\text{err}_{UL}})^{T}(Q_{k}^{err_{UL}})]\\&=\mathbb{E}[ \text{tr}((Q_{k}^{\text{err}_{UL}})(Q_{k}^{err_{UL}})^{T}]\\&=\mathbb{E}[ \text{tr}(X_{k})]
\end{split}
\]
and using the bound in~\cref{app:bounded_trace} and~\cref{app:upper_bound_new_noise} lead to
\begin{align*}
\mathbb{E}[\lVert Q_{k}^{\text{err}_{UL}}\rVert_{2}^{2}]\leq \frac{16\alpha \lvert \mathcal{S}\times \mathcal{A}\rvert^{2}}{d_{\text{min}}(1-\gamma)^{3}}+\lvert \mathcal{S}\times \mathcal{A}\rvert^{2}\lVert Q_{0}^{\text{err}_{UL}}\rVert_{2}^{2}\rho^{2k}
\end{align*}
Taking the square root on both side of the last inequality, using
the subadditivity of the square root function, the Jensen inequality,
and the concavity of the square root function, we have the desired
conclusion.

\end{proof}

To get the upper bound of $Q_{k+1}^{\text{err}_{U}}$, subtracting the lower comparison system $Q_{k+1}^{\text{err}_{UL}}$ from upper comparison system $Q_{k+1}^{\text{err}_{U}}$, the following form can be obtained:
\begin{align}\label{err-u-l}
Q_{k+1}^{\text{err}_{U}}-Q_{k+1}^{\text{err}_{UL}}&=(I-\alpha D)(Q_{k}^{\text{err}_{U}}-Q_{k}^{\text{err}_{UL}})\nonumber+\alpha\gamma 
DP(\Pi_{Q_{k}^{\text{err}_{U}}}Q_{k}^{\text{err}_{U}}-\Pi_{Q_{k}^{*}}Q_{k}^{\text{err}_{UL}})\nonumber\\&=(I-\alpha D)(Q_{k}^{\text{err}_{U}}-Q_{k}^{\text{err}_{UL}})\nonumber+\alpha\gamma DP(
\Pi_{Q_{k}^{\text{err}_{U}}}Q_{k}^{\text{err}_{U}}-\Pi_{Q_{k}^{\text{err}_{U}}}Q_{k}^{\text{err}_{UL}}\nonumber\\&\quad+\Pi_{Q_{k}^{\text{err}_{U}}}Q_{k}^{\text{err}_{UL}}-\Pi_{Q_{k}^{*}}Q_{k}^{\text{err}_{UL}})\nonumber\\&=(I-\alpha D+\alpha\gamma DP\Pi_{Q_{k}^{\text{err}_{U}}})(Q_{k}^{\text{err}_{U}}-Q_{k}^{\text{err}_{UL}})+\alpha\gamma DP(\Pi_{Q_{k}^{\text{err}_{U}}}-\Pi_{Q_{k}^{*}})Q_{k}^{\text{err}_{UL}}
\end{align}
Taking the $\infty$-norm and expectation on (\ref{err-u-l}) yields the bound
\begin{align}\label{eqn:err_upper-upperlower}
\mathbb{E}[\lVert Q_{k+1}^{\text{err}_{U}}-Q_{k+1}^{\text{err}_{UL}}\rVert_{\infty}]&\leq
\lVert I-\alpha D+\alpha\gamma DP\Pi_{Q_{k}^{\text{err}_{U}}}   \rVert_{\infty}\mathbb{E}[\lVert Q_{k}^{\text{err}_{U}}-Q_{k}^{\text{err}_{UL}}\rVert_{\infty}]\nonumber\\&\quad+\lVert\alpha\gamma DP \rVert_{\infty}\lVert\Pi_{Q_{k}^{\text{err}_{U}}}-\Pi_{Q_{k}^{*}}\rVert_{\infty}\mathbb{E}[\lVert Q_{k}^{\text{err}_{UL}}\rVert_{\infty}]\nonumber \\
&\leq\rho\mathbb{E}[\lVert Q_{k}^{\text{err}_{U}}-Q_{k}^{\text{err}_{UL}}\rVert_{\infty}]+\alpha\gamma d_{\text{max}}\mathbb{E}[\lVert Q_{k}^{\text{err}_{UL}}\rVert_{\infty}]\nonumber\\&\leq\rho\mathbb{E}[\lVert Q_{k}^{\text{err}_{U}}-Q_{k}^{\text{err}_{UL}}\rVert_{\infty}]+\alpha\gamma d_{\text{max}}\biggl(\frac{4\alpha^{1/2}\lvert \mathcal{S}\times \mathcal{A}\rvert}{d_{\text{min}}^{1/2}(1-\gamma)^{3/2}}+\lvert \mathcal{S}\times \mathcal{A}\rvert \lVert Q_{0}^{\text{err}_{UL}}\rVert_{2}\rho^{k}\biggl),
\end{align}
where the second inequality is due to \cref{app:system_upper_bound}
and the last inequality is due to  \cref{lem:err_ul_upper_bound}.
Letting $Q_{0}^{\text{err}_{U}}=Q_{0}^{\text{err}_{UL}}$ in~(\ref{eqn:err_upper-upperlower}) and applying the inequality successively result in

\begin{align}\label{eqn:u-ul}
\mathbb{E}[\lVert Q_{k}^{\text{err}_{U}}-Q_{k}^{\text{err}_{UL}}\rVert_{\infty}]&\leq\frac{4\gamma d_{\text{max}} \lvert \mathcal{S} \times \mathcal{A}\rvert\alpha^{1/2}}{d_{\text{min}}^{3/2}(1-\gamma)^{5/2}}+k\rho^{k-1}2\alpha\gamma d_{\text{max}}\frac{\lvert \mathcal{S} \times \mathcal{A}\rvert^{3/2}}{1-\gamma}.
\end{align}

\noindent Using this result, we can obtain the bound of $\mathbb{E}\bigl[\lVert Q_{k}^{\text{err}_{U}}\rVert_{\infty}\bigr]$. Thus, $Q_{k}^{\text{err}_{U}}$ satisfies
\begin{align}\label{eqn:supplementary}
\mathbb{E}\bigl[\lVert Q_{k}^{\text{err}_{U}}\rVert_{\infty}\bigr]&=\mathbb{E}[\lVert
Q_{k}^{\text{err}_{U}}-Q_{k}^{\text{err}_{UL}}+Q_{k}^{\text{err}_{UL}}
\rVert_{\infty}] \nonumber \\
&\leq\mathbb{E}[\lVert Q_{k}^{\text{err}_{U}}-Q_{k}^{\text{err}_{UL}}
\rVert_{\infty}]+\mathbb{E}[\lVert
Q_{k}^{\text{err}_{UL}}
\rVert_{\infty}] \nonumber
\\&\leq\frac{4\gamma d_{\text{max}} \lvert \mathcal{S} \times \mathcal{A}\rvert\alpha^{1/2}}{d_{\text{min}}^{3/2}(1-\gamma)^{5/2}}+k\rho^{k-1}2\alpha\gamma d_{\text{max}}\frac{\lvert \mathcal{S} \times \mathcal{A}\rvert^{3/2}}{1-\gamma}\nonumber+\frac{4\alpha^{1/2}\lvert \mathcal{S}\times \mathcal{A}\rvert}{d_{\text{min}}^{1/2}(1-\gamma)^{3/2}}+\lvert \mathcal{S}\times \mathcal{A}\rvert \lVert Q_{0}^{\text{err}_{UL}}\rVert_{2}\rho^{k}\nonumber\\&\leq\frac{4\gamma d_{\text{max}} \lvert \mathcal{S} \times \mathcal{A}\rvert\alpha^{1/2}}{d_{\text{min}}^{3/2}(1-\gamma)^{5/2}}+k\rho^{k-1}2\alpha\gamma d_{\text{max}}\frac{\lvert \mathcal{S} \times \mathcal{A}\rvert^{3/2}}{1-\gamma}+\frac{4\alpha^{1/2}\lvert \mathcal{S}\times \mathcal{A}\rvert}{d_{\text{min}}^{1/2}(1-\gamma)^{3/2}}+\lvert \mathcal{S}\times \mathcal{A}\rvert^{3/2}\frac{2}{1-\gamma}\rho^{k},
\end{align}
where the second equality is due to (\ref{eqn:u-ul}) and \cref{lem:err_ul_upper_bound} 
and the last inequality is due to the following fact $\lVert Q_{0}^{\text{err}_{UL}}\rVert_{2}\leq \lvert \mathcal{S} \times \mathcal{A} \rvert^{1/2}\lVert Q_{0}^{\text{err}_{UL}}\rVert_{\infty}\leq \lvert \mathcal{S} \times \mathcal{A} \rvert^{1/2}\frac{2}{1-\gamma}$. Because the upper comparison system bounds all trajectory that of original system, we use this bound as the upper bound of the original system.

\subsection{CONVERGENCE OF $Q_{k}^{\mathrm{err}_{L}}$}\label{sec:err-l}
As the next step for the convergence analysis of  $Q_{k}^{\text{err}}$,
let us write the error lower comparison system $Q_{k}^{\text{err}_{L}}$ as follows:
\begin{align*}\label{eqn:disturb_lcs}
Q_{k+1}^{\text{err}_{L}}=(I+\alpha\gamma DP\Pi_{Q_{k}^{B}}-\alpha D) Q_{k}^{\text{err}_{L}}+\alpha w_{k}^{A}-\alpha w_{k}^{B},\quad Q_{0}^{\text{err}_{L}} \in \mathbb{R}^{\lvert \mathcal{S} \rvert \lvert \mathcal{A}\rvert},
\end{align*}
where the stochastic noises, $w_{k}^{A}$ and $w_{k}^{B}$, are identical to those of the original system in~(\ref{double_q_ode}). In the following proposition, we prove that $Q_{k}^{\text{err}_{L}}$ lower bounds $Q_{k}^{\text{err}}$. 
\begin{proposition}\label{prop:disturb_ucs}
Suppose $Q_{0}^{\mathrm{err}_{L}}\leq Q_{0}^{\mathrm{err}}$, where ``$\leq$'' is used as the element-wise inequality. Then, we have
\[
Q_{k}^{\mathrm{err}_{L}}\leq Q_{k}^{\mathrm{err}},
\]
for all $k\geq 0.$
\end{proposition}
\begin{proof}
The proof is completed by an induction argument. Suppose that $Q_{i}^{\text{err}_{L}}\leq Q_{i}^{\text{err}}$ holds for $0\le i \le k$. Then, it follows that
\[
\begin{split}
Q_{k+1}^{\text{err}} &=Q_{k}^{\text{err}}+\alpha\gamma DP\Pi_{Q_{k}^{B}}Q_{k}^{A}-\alpha\gamma DP\Pi_{Q_{k}^{A}} Q_{k}^{B}-\alpha DQ_{k}^{\text{err}}+\alpha w_{k}^{A}-\alpha w_{k}^{B} \\
&\geq Q_{k}^{\text{err}}+\alpha\gamma DP\Pi_{Q_{k}^{B}}Q_{k}^{A}-\alpha\gamma DP\Pi_{Q_{k}^{B}}Q_{k}^{B}-\alpha DQ_{k}^{\text{err}}+\alpha w_{k}^{A}-\alpha w_{k}^{B}\\
&=Q_{k}^{\text{err}}+\alpha\gamma DP\Pi_{Q_{k}^{B}}Q_{k}^{\text{err}}-\alpha DQ_{k}^{\text{err}}+\alpha w_{k}^{A}-\alpha w_{k}^{B}\\
&=(I+\alpha\gamma DP\Pi_{Q_{k}^{B}}-\alpha D)Q_{k}^{\text{err}}+\alpha w_{k}^{A}-\alpha w_{k}^{B}\\
&\geq(I+\alpha\gamma DP\Pi_{Q_{k}^{B}}-\alpha D)Q_{k}^{\text{err}_{L}}+\alpha w_{k}^{A}-\alpha w_{k}^{B}\\&=Q_{k+1}^{\text{err}_{L}},
\end{split}
\]
where the first inequality is due to $\Pi_{Q_{k}^{B}}Q_{k}^{B}\geq \Pi_{Q_{k}^{B}}Q_{k}^{A}$, and the second inequality is due to the hypothesis $Q_{k}^{\text{err}_{L}}\leq Q_{k}^{\text{err}}$ and the fact that the matrix $I+\alpha\gamma DP\Pi_{Q_{k}^{B}}-\alpha D$ is nonnegative, i.e., all elements are nonnegative by \cref{lemma:nonnegative-matrix}. Therefore, $Q_{k+1}^{\text{err}_{L}}\leq Q_{k+1}^{\text{err}}$ holds, and the proof is completed by induction.
\end{proof}
The error lower comparison system switches according to the change of $Q_{k}^{B}$. So it is hard to analyze the stability of the lower comparison system in contrast to (\ref{eqn:disturb_lcs}) which is linear system. To circumvent such a difficulty, we instead study an subtraction system by subtracting the error lower comparison system from the error upper comparison system as follows
\begin{align}\label{supple_another_error}
Q_{k+1}^{\text{err}_{U}}-Q_{k+1}^{\text{err}_{L}}&=(I-\alpha D)(Q_{k}^{\text{err}_{U}}-Q_{k}^{\text{err}_{L}})+\alpha\gamma 
DP(\Pi_{Q_{k}^{\text{err}_{U}}}Q_{k}^{\text{err}_{U}}-\Pi_{Q_{k}^{B}}Q_{k}^{\text{err}_{L}})\nonumber\\&=(I-\alpha D)(Q_{k}^{\text{err}_{U}}-Q_{k}^{\text{err}_{L}})\nonumber+\alpha\gamma DP(
\Pi_{Q_{k}^{\text{err}_{U}}}Q_{k}^{\text{err}_{U}}-\Pi_{Q_{k}^{{B}}}Q_{k}^{\text{err}_{L}}+\Pi_{Q_{k}^{{B}}}Q_{k}^{\text{err}_{U}}-\Pi_{Q_{k}^{B}}Q_{k}^{\text{err}_{U}})\nonumber\\&=(I-\alpha D+\alpha\gamma DP\Pi_{Q_{k}^{B}})(Q_{k}^{\text{err}_{U}}-Q_{k}^{\text{err}_{L}})+\alpha\gamma DP(
\Pi_{Q_{k}^{\text{err}_{U}}}-\Pi_{Q_{k}^{B}})Q_{k}^{\text{err}_{U}},
\end{align}
Here the stochastic noise is canceled out in the error system. The key insight is as follows: if we can prove the stability of the subtraction system, i.e., $Q_{k}^{\text{err}_{U}}-Q_{k}^{\text{err}_{L}}\rightarrow 0$ as $k\rightarrow \infty$, then since $Q_{k}^{\text{err}_{U}}\rightarrow 0$ we have $Q_{k}^{\text{err}_{L}}\rightarrow0$

Taking the $\infty$-norm and expectation on (\ref{supple_another_error}) yields the bound
\begin{align}\label{eqn:2ndhalf}
\mathbb{E}[\lVert Q_{k+1}^{\text{err}_{U}}-Q_{k+1}^{\text{err}_{L}}\rVert_{\infty}]&\leq
\lVert I-\alpha D+\alpha\gamma DP\Pi_{Q_{k}^{{B}}}   \rVert_{\infty}\mathbb{E}[\lVert Q_{k}^{\text{err}_{U}}-Q_{k}^{\text{err}_{L}}\rVert_{\infty}]\nonumber\\&\quad+\lVert\alpha\gamma DP \rVert_{\infty}\lVert\Pi_{Q_{k}^{\text{err}_{U}}}-\Pi_{Q_{k}^{B}}\rVert_{\infty}\mathbb{E}[\lVert Q_{k}^{\text{err}_{U}}\rVert_{\infty}]\nonumber \\
&\leq\rho\mathbb{E}[\lVert Q_{k}^{\text{err}_{U}}-Q_{k}^{\text{err}_{L}}\rVert_{\infty}]+\alpha\gamma d_{\text{max}}\mathbb{E}[\lVert Q_{k}^{\text{err}_{U}}\rVert_{\infty}]\nonumber\\&\leq\rho\mathbb{E}[\lVert Q_{k}^{\text{err}_{U}}-Q_{k}^{\text{err}_{L}}\rVert_{\infty}]\nonumber+\alpha\gamma d_{\text{max}}\biggl(\frac{4\gamma d_{\text{max}} \lvert \mathcal{S} \times \mathcal{A}\rvert\alpha^{1/2}}{d_{\text{min}}^{3/2}(1-\gamma)^{5/2}}\nonumber+k\rho^{k-1}2\alpha\gamma d_{\text{max}}\frac{\lvert\mathcal{S} \times \mathcal{A}\rvert^{3/2}}{1-\gamma}\nonumber\\&
\quad+\frac{4\alpha^{1/2}\lvert \mathcal{S}\times \mathcal{A}\rvert}{d_{\text{min}}^{1/2}(1-\gamma)^{3/2}}+\lvert \mathcal{S}\times \mathcal{A}\rvert^{3/2}\frac{2}{1-\gamma}\rho^{k}\biggl),
\end{align}
where the second inequality is due to \cref{app:system_upper_bound}
and the last inequality is due to  \cref{lem:err_ul_upper_bound}.
Letting $Q_{0}^{\text{err}_{U}}=Q_{0}^{\text{err}_{L}}$ in~(\ref{eqn:2ndhalf}) and applying the inequality successively result in

\begin{align}
\mathbb{E}[\lVert Q_{k}^{\text{err}_{U}}-Q_{k}^{\text{err}_{L}}\rVert_{\infty}]&\leq\rho^{k}\mathbb{E}[\lVert Q_{0}^{\text{err}_{U}}-Q_{0}^{\text{err}_{L}}\rVert_{\infty}]\nonumber+\alpha\gamma d_{\text{max}}\biggl(\frac{\rho^{k}}{1-\rho}\frac{4\gamma d_{\text{max}} \lvert \mathcal{S} \times \mathcal{A}\rvert\alpha^{1/2}}{d_{\text{min}}^{3/2}(1-\gamma)^{5/2}}\nonumber+\frac{\rho^{k}}{1-\rho}\frac{4\alpha^{1/2}\lvert \mathcal{S}\times \mathcal{A}\rvert}{d_{\text{min}}^{1/2}(1-\gamma)^{3/2}}\nonumber\\&\quad+\rho^{k-2}\frac{(k-1)(k-2)}{2}2\alpha\gamma d_{\text{max}}\frac{\lvert \mathcal{S} \times \mathcal{A}\rvert^{3/2}}{1-\gamma}\nonumber+k\rho^{k-1}\frac{2\lvert \mathcal{S}\times \mathcal{A}\rvert^{3/2}}{1-\gamma}\biggl)\nonumber\\&=\rho^{k}\frac{4\gamma^{2} d_{\text{max}}^{2}\lvert \mathcal{S} \times \mathcal{A}\rvert \alpha^{1/2}}{d_{\text{min}}^{5/2}(1-\gamma)^{7/2}}+\rho^{k}\frac{4\gamma d_{\text{max}} \lvert \mathcal{S} \times \mathcal{A}\rvert \alpha^{1/2}}{d_{\text{min}}^{3/2}(1-\gamma)^{5/2}}\nonumber+\rho^{k-2}(k-1)(k-2)\frac{\alpha^{2}\gamma^{2}d_{\text{max}}^{2}\lvert \mathcal{S} \times \mathcal{A}\rvert^{3/2}}{1-\gamma}\nonumber\\&\quad+ k\rho^{k-1}\frac{2\lvert \mathcal{S} \times \mathcal{A} \rvert^{3/2}\alpha\gamma d_{\text{max}}}{1-\gamma},
\end{align}
where the equality is due to \cref{def:d_max_and_d_min}.

\subsection{CONVERGENCE OF $Q_{k}^{\mathrm{err}}$ }\label{sec:external_disturb}

By using upper comparison system and upper-lower comparison system and lower comparison system corresponding to the error system, one can derive the finite-time bound of $Q_{k}^{\text{err}}$.

\begin{lemma}\label{lem:second_half}
For any $k\geq 0$, we have

\begin{align}\label{eqn:second_half}
\mathbb{E}\left[\lVert Q_{k}^{\mathrm{err}}\rVert_{\infty}\right]\leq \frac{8\gamma d_{\mathrm{max}} \lvert \mathcal{S} \times \mathcal{A}\rvert\alpha^{1/2}}{d_{\mathrm{min}}^{5/2}(1-\gamma)^{7/2}}+\frac{8\alpha^{1/2} \lvert \mathcal{S} \times \mathcal{A}\rvert}{d_{\mathrm{min}}^{3/2}(1-\gamma)^{5/2}}+\frac{4\rho^{k-2}k^{2}\alpha\gamma d_{\mathrm{max}}\lvert \mathcal{S} \times \mathcal{A} \rvert^{3/2}}{(1-\gamma)}+\frac{4k\rho^{k-1}\lvert \mathcal{S} \times \mathcal{A} \rvert^{3/2}}{(1-\gamma)}.
\end{align}

\end{lemma}

\begin{proof}
We can get the bound of $Q_{k}^{\text{err}}$ as follows
\begin{align}\label{eqn:err_before_simplified}
\mathbb{E}\left[ \lVert Q_{k}^{\text{err}} \rVert_{\infty}\right] &= \mathbb{E}[ \lVert Q_{k}^{\text{err}}-Q_{k}^{\text{err}_{U}}+Q_{k}^{\text{err}_{U}}\rVert_{\infty}] \nonumber\\
&\leq \mathbb{E}[ \lVert Q_{k}^{\text{err}}-Q_{k}^{\text{err}_{U}} \rVert_{\infty}]+\mathbb{E}[ \lVert Q_{k}^{\text{err}_{U}} \rVert_{\infty}]
\nonumber\\&\leq \mathbb{E}[ \lVert Q_{k}^{\text{err}_{L}}-Q^{\text{err}_{U}} \rVert_{\infty}]+\mathbb{E}[ \lVert Q_{k}^{\text{err}_{U}} \rVert_{\infty}]\nonumber\\&\leq \rho^{k}\frac{4\gamma^{2} d_{\text{max}}^{2}\lvert \mathcal{S} \times \mathcal{A}\rvert \alpha^{1/2}}{d_{\text{min}}^{5/2}(1-\gamma)^{7/2}}+\rho^{k}\frac{4\gamma d_{\text{max}} \lvert \mathcal{S} \times \mathcal{A}\rvert \alpha^{1/2}}{d_{\text{min}}^{3/2}(1-\gamma)^{5/2}}\nonumber+\rho^{k-2}(k-1)(k-2)\frac{\alpha^{2}\gamma^{2}d_{\text{max}}^{2}\lvert \mathcal{S} \times \mathcal{A}\rvert^{3/2}}{1-\gamma}\nonumber\\&\quad+k\rho^{k-1}2\alpha\gamma d_{\text{max}}\frac{\lvert \mathcal{S} \times \mathcal{A} \rvert^{3/2}}{1-\gamma}+\frac{4\gamma d_{\text{max}} \lvert \mathcal{S} \times \mathcal{A}\rvert\alpha^{1/2}}{d_{\text{min}}^{3/2}(1-\gamma)^{5/2}}\nonumber+k\rho^{k-1}2\alpha\gamma d_{\text{max}}\frac{\lvert \mathcal{S} \times \mathcal{A}\rvert^{3/2}}{1-\gamma}\nonumber\\&\quad+\frac{4\alpha^{1/2}\lvert \mathcal{S}\times \mathcal{A}\rvert}{d_{\text{min}}^{1/2}(1-\gamma)^{3/2}}+\lvert \mathcal{S}\times \mathcal{A}\rvert^{3/2}\frac{2}{1-\gamma}\rho^{k}
\end{align}





\begin{align*}
&\rho^{k}\frac{4\gamma^{2} d_{\text{max}}^{2}\lvert \mathcal{S} \times \mathcal{A}\rvert \alpha^{1/2}}
{d_{\text{min}}^{5/2}(1-\gamma)^{7/2}}
+\rho^{k}\frac{4\gamma d_{\text{max}} \lvert \mathcal{S} \times \mathcal{A}\rvert \alpha^{1/2}}
{d_{\text{min}}^{3/2}(1-\gamma)^{5/2}}
\leq
\frac{8\gamma d_{\text{max}} \lvert \mathcal{S} \times \mathcal{A}\rvert\alpha^{1/2}}
{d_{\text{min}}^{5/2}(1-\gamma)^{7/2}},
\\[1ex]
&\frac{4\gamma d_{\text{max}} \lvert \mathcal{S} \times \mathcal{A}\rvert\alpha^{1/2}}
{d_{\text{min}}^{3/2}(1-\gamma)^{5/2}}
+\frac{4\alpha^{1/2}\lvert \mathcal{S}\times \mathcal{A}\rvert}
{d_{\text{min}}^{1/2}(1-\gamma)^{3/2}}
\leq
\frac{8\alpha^{1/2} \lvert \mathcal{S} \times \mathcal{A}\rvert}
{d_{\text{min}}^{3/2}(1-\gamma)^{5/2}},
\\[1ex]
&\rho^{k-2}(k-1)(k-2)
\frac{\alpha^{2}\gamma^{2}d_{\text{max}}^{2}\lvert \mathcal{S} \times \mathcal{A}\rvert^{3/2}}
{1-\gamma}
+2k\rho^{k-1}\alpha\gamma d_{\text{max}}
\frac{\lvert \mathcal{S} \times \mathcal{A} \rvert^{3/2}}
{1-\gamma}
\leq
\frac{4\rho^{k-2}k^{2}\alpha\gamma d_{\text{max}}
\lvert \mathcal{S} \times \mathcal{A} \rvert^{3/2}}
{1-\gamma},
\\[1ex]
&\frac{2\rho^{k}\lvert \mathcal{S}\times \mathcal{A}\rvert^{3/2}}
{1-\gamma}
+2k\rho^{k-1}\alpha\gamma d_{\text{max}}
\frac{\lvert \mathcal{S} \times \mathcal{A}\rvert^{3/2}}
{1-\gamma}
\leq
\frac{4k\rho^{k-1}\lvert \mathcal{S} \times \mathcal{A} \rvert^{3/2}}
{1-\gamma}.
\end{align*}
Then we can get a simplified form as
\begin{align*}
\mathbb{E}\left[\lVert Q_{k}^{\text{err}}\rVert_{\infty}\right]&\leq \frac{8\gamma d_{\text{max}} \lvert \mathcal{S} \times \mathcal{A}\rvert\alpha^{1/2}}{d_{\text{min}}^{5/2}(1-\gamma)^{7/2}}+\frac{8\alpha^{1/2} \lvert \mathcal{S} \times \mathcal{A}\rvert}{d_{\text{min}}^{3/2}(1-\gamma)^{5/2}}+\frac{4\rho^{k-2}k^{2}\alpha\gamma d_{\text{max}}\lvert \mathcal{S} \times \mathcal{A} \rvert^{3/2}}{(1-\gamma)}+\frac{4k\rho^{k-1}\lvert \mathcal{S} \times \mathcal{A} \rvert^{3/2}}{(1-\gamma)}.
\end{align*}
\end{proof}
\noindent This completes the finite-time error bound for $Q_{k}^{\text{err}}$ by
combining the upper, lower, and upper--lower comparison systems.
\section{Detailed analysis result of convergence of SDQ (Remaining part) }\label{remaining_part}

In this section, convergence of SDQ will be studied based on the results in~\ref{sec:disturb}. In~\ref{sec:disturb}, a bound on $Q_{k}^{\text{err}}$ has been obtained  in \cref{lem:second_half}. 
To compete the proof, the following three steps remain:
\begin{itemize}
    \item Step~3: Using the bound on $Q_{k}^{\text{err}}$ and the linear structures of $Q_{k}^{A_L} - Q^*$ and $Q_{k}^{B_L} - Q^*$, a finite-time error bounds on $Q_{k}^{A_L} - Q^*$ and $Q_{k}^{B_L} - Q^*$ can be derived.
    \item Step~4: By obtaining a subtraction system which can be obtained by subtracting the lower comparison system from upper comparison system, the convergence of $Q_{k}^{A_U} - Q^*$ and $Q_{k}^{B_U} - Q^*$ can be shown.
    \item Step~5: Next, combining the result from Step 4  with the upper comparison system $Q_{k}^{A_U} - Q^*$ and $Q_{k}^{B_U} - Q^*$, we can finally obtain the finite-time error bound on the iterates of SDQ. 
\end{itemize}

\subsection{Proof of Proposition \ref{prop:ucs} (Upper comparison system)}\label{sec:prop-ucs}

Using the dynamic system equation~({\ref{eqn:original}}), we have
\[
\begin{split}
Q_{k+1}^{A}-Q^{*}&=Q_{k}^{A}-Q^{*}+\alpha D\{\gamma P\Pi_{Q_{k}^{B}}Q_{k}^{A}-\gamma P\Pi_{Q^{*}}Q^{*}-Q_{k}^{A}+Q^{*}\}+\alpha w_{k}^{A}\\
&\leq Q_{k}^{A}-Q^{*}+\alpha D\{\gamma P\Pi_{Q_{k}^{B}}Q_{k}^{A}-\gamma P\Pi_{Q_{k}^{B}}Q^{*}-Q_{k}^{A}+Q^{*}\}+\alpha w_{k}^{A}\\
&=(I+\alpha\gamma DP\Pi_{Q_{k}^{B}}-\alpha D)(Q_{k}^{A}-Q^{*})+\alpha w_{k}^{A}\\
&\leq (I+\alpha\gamma DP\Pi_{Q_{k}^{B}}-\alpha D)(Q_{k}^{A_{U}}-Q^{*})+\alpha w_{k}^{A}\\
& = Q_{k+1}^{A_{U}}-Q^{*} ,
\end{split}
\]
where the first inequality is due to $\Pi_{Q^{*}}Q^{*}\geq \Pi_{Q_{k}^{B}}Q^{*}$, and the second inequality is due to the hypothesis $Q_{k}^{A_{U}}-Q^{*}\geq Q_{k}^{A}-Q^{*}$ and the fact that the matrix $I+\alpha\gamma DP\Pi_{Q_{k}^{B}}-\alpha D$ is nonnegative, i.e., all elements are nonnegative by \cref{lemma:nonnegative-matrix}.
Therefore, by induction argument, one concludes $Q_{k}^{A_{U}}-Q^{*}\geq Q_{k}^{A}-Q^{*}$ for all $k\geq 0$. 
The proof of the second inequality follows similar lines. This completes the proof.

\subsection{Proof of Proposition \ref{prop:lcs} (Lower comparison system)}\label{sec:prop-lcs}

Using the dynamic system equation~({\ref{eqn:original}}), we have
\[
\begin{split}
Q_{k+1}^{A}-Q^{*}&=Q_{k}^{A}-Q^{*}+\alpha D\{\gamma P\Pi_{Q_{k}^{B}}Q_{k}^{A}-\gamma P\Pi_{Q^{*}}Q^{*}-Q_{k}^{A}+Q^{*}\}+\alpha w_{k}^{A}
\\&=(I-\alpha D)(Q_{k}^{A}-Q^{*})+\alpha D\{\gamma P\Pi_{Q_{k}^{B}}Q_{k}^{B}-\gamma P\Pi_{Q^{*}}Q^{*}+\gamma P\Pi_{Q_{k}^{B}}(Q_{k}^{A}-Q_{k}^{B})\}+\alpha w_{k}^{A}
\\&\geq (I-\alpha D)(Q_{k}^{A}-Q^{*})+\alpha D\{\gamma P\Pi_{Q^{*}}Q_{k}^{B}-\gamma P\Pi_{Q^{*}}Q^{*}+\gamma P\Pi_{Q_{k}^{B}}(Q_{k}^{A}-Q_{k}^{B})\}+\alpha w_{k}^{A}\\&=(I+\alpha\gamma DP\Pi_{Q^{*}}-\alpha D)(Q_{k}^{A}-Q^{*})+\alpha\gamma DP(\Pi_{Q_{k}^{B}}-\Pi_{Q^{*}})(Q_{k}^{A}-Q_{k}^{B})+\alpha w_{k}^{A}\\&\geq (I+\alpha\gamma DP\Pi_{Q^{*}}-\alpha D)(Q_{k}^{A_{L}}-Q^{*})+\alpha\gamma DP(\Pi_{Q_{k}^{B}}-\Pi_{Q^{*}})(Q_{k}^{A}-Q_{k}^{B})+\alpha w_{k}^{A}\\
&=Q_{k+1}^{A_{L}}-Q^{*},
\end{split}
\]
where the first inequality is due to $\Pi_{Q^{B}}Q^{B}\geq \Pi_{Q^{*}}Q^{B}$, and the second inequality is due to the hypothesis $Q_{k}^{A_{L}}-Q^{*}\leq Q_{k}^{A}-Q^{*}$. Therefore, by induction argument, one concludes $Q_{k}^{A_{L}}-Q^{*}\leq Q_{k}^{A}-Q^{*}$ for all $k\geq 0$. And the second inequality is due to the hypothesis $Q_{k}^{{A_{L}}}\leq Q_{k}^{A}$ and the fact that the matrix $I+\alpha\gamma DP\Pi_{Q{*}}-\alpha D$ is nonnegative, i.e., all elements are nonnegative by \cref{lemma:nonnegative-matrix}.
The proof of the second inequality follows lines similar to the first proof.
This completes the proof.

\subsection{Convergence of the lower comparison system}

The lower comparison system in~(\ref{eq:lower-comaprison-system})
can be divided into the linear parts with stochastic noises, $(I+\alpha\gamma DP\Pi_{Q^{*}}-\alpha D)(Q_{k}^{A_{L}}-Q^{*})+\alpha w_{k}^{A}$ and $(I+\alpha\gamma DP\Pi_{Q^{*}}-\alpha D)(Q_{k}^{B_{L}}-Q^{*})+\alpha w_{k}^{B}$, and the external disturbance parts, $\alpha\gamma DP(\Pi_{Q_{k}^{B}}-\Pi_{Q^{*}})(Q_{k}^{A}-Q_{k}^{B})$ and $\alpha\gamma DP(\Pi_{Q_{k}^{*}}-\Pi_{Q_{k}^{A}})(Q_{k}^{A}-Q_{k}^{B})$. As proved in~\ref{sec:external_disturb}, the external disturbances are bounded. Using this fact, one can prove the finite-time error bounds of the linear part with stochastic noise as $(I+\alpha\gamma DP\Pi_{Q^{*}}-\alpha D)(Q_{k}^{A_{L}}-Q^{*})+\alpha w_{k}^{A}$ and $(I+\alpha\gamma DP\Pi_{Q^{*}}-\alpha D)(Q_{k}^{B_{L}}-Q^{*})+\alpha w_{k}^{B}$. 

\begin{theorem}\label{thm:lcs_upper_bound}
For any $k\geq 0$, we have
\begin{align}\label{eqn:lcs_upper_bound}
\mathbb{E}[\lVert Q_{k}^{A_{L}}-Q^{*} \rVert_{\infty}]&\leq \frac{16\gamma d_{\mathrm{max}}\lvert \mathcal{S}\times \mathcal{A}\rvert\alpha^{1/2}}{d_{\mathrm{min}}^{7/2}(1-\gamma)^{9/2}}+\frac{24\rho^{k-3}k^{3}\lvert \mathcal{S}\times \mathcal{A}\rvert^{3/2}}{(1-\gamma)}+\frac{4\alpha^{1/2}\lvert \mathcal{S}\times \mathcal{A}\rvert}{d_{\mathrm{min}}^{1/2}(1-\gamma)^{3/2}}.
\end{align}
The same bound holds for $Q_{k}^{B_{L}}-Q^{*}$. 
\end{theorem}

\begin{proof}
First of all, note that~(\ref{eq:lower-comaprison-system}) can be written by

\begin{align*}
Q_k^{A_L} - Q^*
&= (I + \alpha\gamma DP\Pi_{Q^*} - \alpha D)^k\,(Q_0^{A_L} - Q^*)
  + \\&\quad\underbrace{\sum_{j=0}^{k-1}\alpha\,(I + \alpha\gamma DP\Pi_{Q^*} - \alpha D)^{(k-1)-j}\,w_j^A}_{=:(*)}
  + \underbrace{
      \begin{aligned}
        \alpha\gamma DP
        \sum_{j=0}^{k-1}(I + \alpha\gamma DP\Pi_{Q^*} - \alpha D)^{(k-1)-j}\\
        \quad\quad\times\,(\Pi_{Q_j^B} - \Pi_{Q^*})\,(Q_j^A - Q_j^B)
      \end{aligned},
    }_{=:(**)}
\end{align*}
where $(*)$ reflects the effect of the stochastic noise $w_j^A$ and $(**)$ corresponds to the effect of the disturbance $Q_j^A - Q_j^B$. 
Taking the $\infty$-norm on the right-hand side of the above equation leads to
\begin{align*}
\lVert Q_{k}^{A_{L}}-Q^{*}\rVert_{\infty} =&  \biggl\lVert (I+\alpha\gamma DP\Pi_{Q^{*}}-\alpha D)^{k}(Q_{0}^{A_{L}}-Q^{*})+\sum_{j=0}^{k-1}\alpha (I+\alpha\gamma DP\Pi_{Q^{*}}-\alpha D)^{(k-1)-j}w_{j}^{A}\\
&\quad+\alpha\gamma DP\sum_{j=0}^{k-1}(I+\alpha\gamma DP\Pi_{Q^{*}}-\alpha D)^{(k-1)-j}(\Pi_{Q_{j}^{B}}-\Pi_{Q^{*}})(Q_{j}^{A}-Q_{j}^{B})\biggr\rVert_{\infty}\\
\leq&\biggl\lVert(I+\alpha\gamma DP\Pi_{Q^{*}}-\alpha D)^{k}(Q_{0}^{A_{L}}-Q^{*})+\sum_{j=0}^{k-1}\alpha (I+\alpha\gamma DP\Pi_{Q^{*}}-\alpha D)^{(k-1)-j}w_{j}^{A}\biggr\rVert_{\infty}\\
&\quad+\biggl\lVert\alpha\gamma DP\sum_{j=0}^{k-1}(I+\alpha\gamma DP\Pi_{Q^{*}}-\alpha D)^{(k-1)-j}(\Pi_{Q_{j}^{B}}-\Pi_{Q^{*}})(Q_{j}^{A}-Q_{j}^{B})\biggr\rVert_{\infty}\\
\leq& \underbrace{\biggl\lVert (I + \alpha \gamma DP{\Pi _{{Q^*}}} - \alpha D)^k(Q_0^{{A_L}} - {Q^*})}_{:=(*)} \underbrace{+ \sum\limits_{j = 0}^{k - 1} \alpha  {{(I + \alpha \gamma DP{\Pi _{{Q^*}}} - \alpha D)}^{(k - 1) - j}}w_j^A\biggr\rVert_\infty}_{:=(**)} \\
&+\biggl\lVert\alpha\gamma DP\biggr\rVert_{\infty}\sum_{j=0}^{k-1}\rho^{(k-1)-j}\biggl\lVert(\Pi_{Q_{j}^{B}}-\Pi_{Q^{*}})\biggr\rVert_{\infty}\biggl\lVert(Q_{j}^{A}-Q_{j}^{B})\biggr\rVert_{\infty},
\end{align*}
where $(*)$ and $(**)$ in the second inequality corresponds to the solution of~(\ref{eqn:linear_system}) with $x_k = Q_k^{{A_L}} - {Q^*}$ and $w_k = w_k^A$, we can apply the bound given in ~\cref{coro:finite_time_bound}. Moreover, applying~\cref{lem:err_ul_upper_bound}, one gets 
\begin{align*}
\lVert Q_{k}^{A_{L}}-Q^{*}\rVert_{\infty} \leq&\frac{4\alpha^{1/2}\lvert \mathcal{S}\times \mathcal{A}\rvert}{d_{\text{min}}^{1/2}(1-\gamma)^{3/2}}+\lvert \mathcal{S}\times \mathcal{A}\rvert^{3/2}\frac{2}{1-\gamma}\rho^{k}+\alpha\gamma d_{\text{max}}\sum_{j=0}^{k-1}\rho^{(k-1)-j}\biggl\lVert(\Pi_{Q_{j}^{B}}-\Pi_{Q^{*}})\biggr\rVert_{\infty}\biggl\lVert(Q_{j}^{A}-Q_{j}^{B})\biggr\rVert_{\infty},
\end{align*}
Combining this with (\ref{eqn:second_half}), we can obtain the following form: 
\begin{align}\label{eqn:lower_bound}
\mathbb{E}[\lVert Q_{k}^{A_{L}}-Q^{*} \rVert_{\infty}]&\leq\frac{8\gamma^{2}d_{\text{max}}^{2}\lvert \mathcal{S}\times \mathcal{A}\rvert \alpha^{1/2}}{d_{\text{min}}^{7/2}(1-\gamma)^{9/2}}+\frac{8\gamma d_{\text{max}}\lvert \mathcal{S}\times \mathcal{A}\rvert \alpha^{1/2}}{d_{\text{min}}^{5/2}(1-\gamma)^{7/2}}\nonumber\\&\quad+\rho^{k-1}\biggl(\frac{(k-1)k(2k-1)}{6}\biggl)\frac{4\rho^{-2}\alpha^{2}\gamma^{2}d_{\text{max}}^{2}\lvert \mathcal{S}\times \mathcal{A}\rvert^{3/2}}{(1-\gamma)}\nonumber\\&\quad+\rho^{k-1}\biggl(\frac{(k-1)k}{2}\biggl)\frac{4\rho^{-1}\alpha\gamma d_{\text{max}}\lvert \mathcal{S}\times \mathcal{A}\rvert^{3/2}}{(1-\gamma)}+\frac{4\alpha^{1/2}\lvert \mathcal{S}\times \mathcal{A}\rvert}{d_{\text{min}}^{1/2}(1-\gamma)^{3/2}}+\frac{2\rho^{k}\lvert \mathcal{S}\times \mathcal{A}\rvert^{3/2}}{(1-\gamma)}.
\end{align}
Then we group some terms of (\ref{eqn:lower_bound}) as
\begin{align*}
\frac{8\gamma^{2}d_{\text{max}}^{2}\lvert \mathcal{S}\times \mathcal{A}\rvert \alpha^{1/2}}{d_{\text{min}}^{7/2}(1-\gamma)^{9/2}}+\frac{8\gamma d_{\text{max}}\lvert \mathcal{S}\times \mathcal{A}\rvert \alpha^{1/2}}{d_{\text{min}}^{5/2}(1-\gamma)^{7/2}}\leq 2\biggl(\frac{8\gamma d_{\text{max}}\lvert \mathcal{S}\times \mathcal{A}\rvert\alpha^{1/2}}{d_{\text{min}}^{7/2}(1-\gamma)^{9/2}}\biggl)
\end{align*}
Other terms also can be grouped as follows
\begin{gather*}
\rho^{k-1}\biggl(\frac{(k-1)k(2k-1)}{6}\biggl)\frac{4\rho^{-2}\alpha^{2}\gamma^{2}d_{\text{max}}^{2}\lvert \mathcal{S}\times \mathcal{A}\rvert^{3/2}}{(1-\gamma)}\nonumber+\rho^{k-1}\biggl(\frac{(k-1)k}{2}\biggl)\frac{4\rho^{-1}\alpha\gamma d_{\text{max}}\lvert \mathcal{S}\times \mathcal{A}\rvert^{3/2}}{(1-\gamma)}\\+\frac{2\rho^{k}\lvert \mathcal{S}\times \mathcal{A}\rvert^{3/2}}{(1-\gamma)}\leq 3\biggl(\frac{\rho^{k-3}2k^{3}4\lvert \mathcal{S}\times \mathcal{A}\rvert^{3/2}}{(1-\gamma)}\biggl)
\end{gather*}
Then we can get the simplified form as follows
\begin{align}
\mathbb{E}[\lVert Q_{k}^{A_{L}}-Q^{*} \rVert_{\infty}]&\leq \frac{16\gamma d_{\text{max}}\lvert \mathcal{S}\times \mathcal{A}\rvert\alpha^{1/2}}{d_{\text{min}}^{7/2}(1-\gamma)^{9/2}}+\frac{24\rho^{k-3}k^{3}\lvert \mathcal{S}\times \mathcal{A}\rvert^{3/2}}{(1-\gamma)}+\frac{4\alpha^{1/2}\lvert \mathcal{S}\times \mathcal{A}\rvert}{d_{\text{min}}^{1/2}(1-\gamma)^{3/2}}.
\end{align}

\end{proof}
\subsection{Convergence of the upper comparison system}\label{sec:thm-error}
While the lower comparison system can be analyzed using stochastic linear system characteristic, it is relevantly harder to establish the finite-time error bounds of the upper comparison system because the upper comparison system is a switching system. Therefore, instead of directly finding the finite-time bounds of the upper comparison system, we will use a subtraction system that can be obtained by subtracting the lower comparison system (\ref{eq:lower-comaprison-system}) from the upper comparison system (\ref{eqn:ups}) as follows: 
\begin{align}\label{eqn:subtracting-system}
Q_{k+1}^{A_{U}}-Q_{k+1}^{A_{L}}&=(I-\alpha D)(Q_{k}^{A_{U}}-Q_{k}^{A_{L}})
\nonumber+\alpha\gamma DP\{\Pi_{Q_{k}^{B}}(Q_{k}^{A_{U}}-Q^{*})\nonumber-\Pi_{Q^{*}}(Q_{k}^{A_{L}}-Q^{*})\}
\nonumber\\&\quad-\alpha\gamma DP(\Pi_{Q_{k}^{B}}-\Pi_{Q^{*}})(Q_{k}^{A}-Q_{k}^{B}),\quad Q_{0}^{A_{U}}-Q_{0}^{A_{L}} \in \mathbb{R}^{\lvert \mathcal{S} \rvert \lvert \mathcal{A}\rvert},\nonumber\\
Q_{k+1}^{B_{U}}-Q_{k+1}^{B_{L}}&=(I-\alpha D)(Q_{k}^{B_{U}}-Q_{k}^{B_{L}})\nonumber+\alpha\gamma DP\{\Pi_{Q_{k}^{A}}(Q_{k}^{B_{U}}-Q^{*})\nonumber-\Pi_{Q^{*}}(Q_{k}^{B_{L}}-Q^{*})\}
\nonumber\\&\quad-\alpha\gamma DP(\Pi_{Q_{k}^{*}}-\Pi_{Q^{A}})(Q_{k}^{A}-Q_{k}^{B}),\quad Q_{0}^{B_{U}}-Q_{0}^{B_{L}} \in \mathbb{R}^{\lvert \mathcal{S} \rvert \lvert \mathcal{A}\rvert},
\end{align}
where the stochastic noises, $w_{k}^{A}$ and $w_{k}^{B}$, are canceled out. If one can prove the stability of the subtraction system, i.e., $Q_{k}^{A_{U}}-Q_{k}^{A_{L}}\rightarrow 0$ and $Q_{k}^{B_{U}}-Q_{k}^{B_{L}}\rightarrow 0$ as $k\rightarrow \infty$ then since $Q_{k}^{A_{L}}\rightarrow Q^{*}$ and $Q_{k}^{B_{L}}\rightarrow Q^{*}$  as $k\rightarrow \infty$, one can prove $Q_{k}^{A_{U}}\rightarrow Q^{*}$ and $Q_{k}^{B_{U}}\rightarrow Q^{*}$as $k\rightarrow \infty$ as well. 
In the following, we prove the finite-time error bound of the subtraction system.
\begin{theorem}\label{thm:error_upper_bound}
For any $k\geq 0$, we have
\begin{align}\label{eqn:error_upper_bound}
\mathbb{E}[\lVert Q_{k}^{A_{U}}-Q_{k}^{A_{L}} \rVert_{\infty}]
&\leq \frac{40\gamma d_{\mathrm{max}}\lvert \mathcal{S}\times \mathcal{A}\rvert \alpha^{1/2}}{d_{\mathrm{min}}^{9/2}(1-\gamma)^{11/2}}+\frac{20\rho^{k-4}k^{4}\alpha\gamma d_{\mathrm{max}}\lvert \mathcal{S}\times \mathcal{A}\rvert^{3/2}}{1-\gamma}.
\end{align}
\end{theorem}

\begin{proof}
The upper bound of $Q_{k+1}^{A_{U}}-Q_{k+1}^{A_{L}}$ can be presented as following using~(\ref{eqn:subtracting-system})
\begin{align}\label{ucs_upper_bound}
Q_{k+1}^{A_{U}}-Q_{k+1}^{A_{L}} &=(I-\alpha D)(Q_{k}^{A_{U}}-Q_{k}^{A_{L}})
+\alpha\gamma DP\{\Pi_{Q_{k}^{B}}(Q_{k}^{A_{U}}-Q^{*})\nonumber\\&\quad-\Pi_{Q^{*}}(Q_{k}^{A_{L}}-Q^{*})\}
\nonumber-\alpha\gamma DP(\Pi_{Q_{k}^{B}}-\Pi_{Q^{*}})(Q_{k}^{A}-Q_{k}^{B})
\nonumber\\&= (I-\alpha D)(Q_{k}^{A_{U}}-Q_{k}^{A_{L}})
+\alpha\gamma DP\{\Pi_{Q_{k}^{B}}(Q_{k}^{A_{U}}-Q^{*})\nonumber\\&\quad+\Pi_{Q_{k}^{B}}(Q_{k}^{A_{L}}-Q^{*})\nonumber-\Pi_{Q_{k}^{B}}(Q_{k}^{A_{L}}-Q^{*})-\Pi_{Q^{*}}(Q_{k}^{A_{L}}-Q^{*})\}\nonumber\\&\quad-\alpha\gamma DP(\Pi_{Q_{k}^{B}}-\Pi_{Q^{*}})(Q_{k}^{A}-Q_{k}^{B})\nonumber\\&=(I-\alpha D)(Q_{k}^{A_{U}}-Q_{k}^{A_{L}})+\alpha\gamma DP\Pi_{Q_{k}^{B}}(Q_{k}^{A_{U}}-Q_{k}^{A_{L}})\nonumber+\alpha\gamma DP(Q_{k}^{A_{L}}-Q^{*})(\Pi_{Q_{k}^{B}}-\Pi_{Q^{*}})\nonumber\\&\quad-\alpha\gamma DP(\Pi_{Q_{k}^{B}}-\Pi_{Q^{*}})(Q_{k}^{A}-Q_{k}^{B})\nonumber\\&=(I+\alpha\gamma DP\Pi_{Q_{k}^{B}}-\alpha D)(Q_{k}^{A_{U}}-Q_{k}^{A_{L}})\nonumber+\alpha\gamma DP(Q_{k}^{A_{L}}-Q^{*})(\Pi_{Q_{k}^{B}}-\Pi_{Q^{*}})\nonumber\\&\quad-\alpha\gamma DP(\Pi_{Q_{k}^{B}}-\Pi_{Q^{*}})(Q_{k}^{A}-Q_{k}^{B})
\end{align}
Taking the $\infty$-norm on (\ref{ucs_upper_bound}) and applying the inequality successively result in 
\begin{align}\label{ucs_upper_bound_compact_form}
\lVert Q_{k+1}^{A_{U}}-Q_{k+1}^{A_{L}} \rVert_{\infty}&\leq\lVert I+\alpha\gamma DP\Pi_{Q_{k}^{B}}-\alpha D\rVert_{\infty}\lVert Q_{k}^{A_{U}}-Q_{k}^{A_{L}} \rVert_{\infty}\nonumber\\&\quad+\biggl\lVert\alpha\gamma DP\biggr\rVert_{\infty}\sum_{j=0}^{k-1}\rho^{(k-1)-j}\biggl\lVert(\Pi_{Q_{j}^{B}}-\Pi_{Q^{*}})\biggr\rVert_{\infty}\biggl\lVert(Q_{j}^{A_{L}}-Q^{*})\biggr\rVert_{\infty}\nonumber\\&\quad+
\biggl\lVert\alpha\gamma DP\biggr\rVert_{\infty}\sum_{j=0}^{k-1}\rho^{(k-1)-j}\biggl\lVert(\Pi_{Q_{j}^{B}}-\Pi_{Q^{*}})\biggr\rVert_{\infty}\biggl\lVert(Q_{j}^{A}-Q_{j}^{B})\biggr\rVert_{\infty}
\end{align}
Assuming $Q_{0}^{A_{U}}=Q_{0}^{A_{L}}$ and taking expectation of (\ref{ucs_upper_bound_compact_form}) lead to
\begin{align}\label{eqn:subtract_system_bound}
\mathbb{E}[\lVert Q_{k}^{A_{U}}-Q_{k}^{A_{L}} \rVert_{\infty}]
&\leq \frac{8\gamma^{3}d_{\text{max}}^{3}\lvert \mathcal{S}\times \mathcal{A}\rvert \alpha^{1/2}}{d_{\text{min}}^{9/2}(1-\gamma)^{11/2}}+\frac{8\gamma^{2}d_{\text{max}}^{2}\lvert \mathcal{S}\times \mathcal{A}\rvert \alpha^{1/2}}{d_{\text{min}}^{7/2}(1-\gamma)^{9/2}}\nonumber+\frac{4\gamma d_{\text{max}}\lvert\mathcal{S}\times \mathcal{A}\rvert \alpha^{1/2}}{d_{\text{min}}^{3/2}(1-\gamma)^{5/2}}\nonumber\\&\quad+\rho^{k-1}k\frac{2\gamma d_{\text{max}}\lvert \mathcal{S}\times \mathcal{A}\rvert^{3/2} \alpha}{(1-\gamma)}\nonumber+\rho^{k-1}\frac{(k-1)^{2}k(k-2)}{2}\frac{1}{6}\frac{4\rho^{-3}\alpha^{3}\gamma d_{\text{max}}^{3}\lvert \mathcal{S}\times \mathcal{A}\rvert^{3/2}}{(1-\gamma)}\nonumber\\&\quad+\rho^{k-1}\frac{4\rho^{-2}\alpha^{2}\gamma^{2}d_{\text{max}}^{2}\lvert \mathcal{S}\times \mathcal{A}\rvert^{3/2}}{1-\gamma}\frac{1}{2}\frac{k(k-1)(k-2)}{3}\nonumber\\&\quad+\frac{8\gamma^{2}d_{\text{max}}^{2}\lvert \mathcal{S}\times \mathcal{A}\rvert \alpha^{1/2}}{d_{\text{min}}^{7/2}(1-\gamma)^{9/2}}+\frac{8\gamma d_{\text{max}}\lvert\mathcal{S}\times \mathcal{A}\rvert \alpha^{1/2}}{d_{\text{min}}^{5/2}(1-\gamma)^{7/2}}\nonumber\\&\quad+\rho^{k-1}\biggl(\frac{(k-1)k(2k-1)}{6}\biggl)\frac{4\rho^{-2}\alpha^{2}\gamma^{2}d_{\text{max}}^{2}\lvert \mathcal{S}\times \mathcal{A}\rvert^{3/2}}{(1-\gamma)}\nonumber\\&\quad+\rho^{k-1}\biggl(\frac{(k-1)k}{2}\biggl)\frac{4\rho^{-1}\alpha\gamma d_{\text{max}}\lvert \mathcal{S}\times \mathcal{A}\rvert^{3/2}}{(1-\gamma)}
\end{align}
We group some terms of (\ref{eqn:subtract_system_bound}) as follows
\begin{align*}
&\frac{8\gamma^{3}d_{\text{max}}^{3}\lvert \mathcal{S}\times \mathcal{A}\rvert \alpha^{1/2}}
      {d_{\text{min}}^{9/2}(1-\gamma)^{11/2}}
+\frac{8\gamma^{2}d_{\text{max}}^{2}\lvert \mathcal{S}\times \mathcal{A}\rvert \alpha^{1/2}}
      {d_{\text{min}}^{7/2}(1-\gamma)^{9/2}}
+\frac{4\gamma d_{\text{max}}\lvert \mathcal{S}\times \mathcal{A}\rvert \alpha^{1/2}}
      {d_{\text{min}}^{3/2}(1-\gamma)^{5/2}
}
+\frac{8\gamma^{2}d_{\text{max}}^{2}\lvert \mathcal{S}\times \mathcal{A}\rvert \alpha^{1/2}}
      {d_{\text{min}}^{7/2}(1-\gamma)^{9/2}}
\\
&\quad
+\frac{8\gamma d_{\text{max}}\lvert \mathcal{S}\times \mathcal{A}\rvert \alpha^{1/2}}
      {d_{\text{min}}^{5/2}(1-\gamma)^{7/2}}
\\
&\leq
5\biggl(
\frac{8\gamma d_{\text{max}}\lvert \mathcal{S}\times \mathcal{A}\rvert \alpha^{1/2}}
     {d_{\text{min}}^{9/2}(1-\gamma)^{11/2}}
\biggr).
\end{align*}
Also we group other remaining terms as follows
\begin{align*}
&\rho^{k-1} k
 \frac{2\gamma d_{\text{max}}\lvert \mathcal{S}\times \mathcal{A}\rvert^{3/2}\alpha}
      {1-\gamma}
+ \rho^{k-1}
  \frac{(k-1)^{2}k(k-2)}{12}
  \frac{4\rho^{-3}\alpha^{3}\gamma d_{\text{max}}^{3}
        \lvert \mathcal{S}\times \mathcal{A}\rvert^{3/2}}
       {1-\gamma}
\\
&\quad
+ \rho^{k-1}
  \frac{k(k-1)(k-2)}{6}
  \frac{4\rho^{-2}\alpha^{2}\gamma^{2}d_{\text{max}}^{2}
        \lvert \mathcal{S}\times \mathcal{A}\rvert^{3/2}}
       {1-\gamma}
+ \rho^{k-1}
  \frac{(k-1)k(2k-1)}{6}
  \frac{4\rho^{-2}\alpha^{2}\gamma^{2}d_{\text{max}}^{2}
        \lvert \mathcal{S}\times \mathcal{A}\rvert^{3/2}}
       {1-\gamma}
\\
&\quad
+ \rho^{k-1}
  \frac{(k-1)k}{2}
  \frac{4\rho^{-1}\alpha\gamma d_{\text{max}}
        \lvert \mathcal{S}\times \mathcal{A}\rvert^{3/2}}
       {1-\gamma}
\\
&\leq
5\biggl(
\frac{4\rho^{k-4}k^{4}\alpha\gamma d_{\text{max}}
      \lvert \mathcal{S}\times \mathcal{A}\rvert^{3/2}}
     {1-\gamma}
\biggr).
\end{align*}
Then, we can get the following simplified form
\begin{align}
\mathbb{E}[\lVert Q_{k}^{A_{U}}-Q_{k}^{A_{L}} \rVert_{\infty}]
&\leq \frac{40\gamma d_{\text{max}}\lvert \mathcal{S}\times \mathcal{A}\rvert \alpha^{1/2}}{d_{\text{min}}^{9/2}(1-\gamma)^{11/2}}+\frac{20\rho^{k-4}k^{4}\alpha\gamma d_{\text{max}}\lvert \mathcal{S}\times \mathcal{A}\rvert^{3/2}}{1-\gamma}.
\end{align}

\end{proof}
\subsection{Proof of Theorem \ref{thm:final-theorem} (Finite-time error bound of SDQ)}
\label{sec:final-theorem}
We can use the fact
\[
\begin{split}
\mathbb{E}\left[ \lVert Q_{k}^{A}-Q^{*} \rVert_{\infty}\right] &= \mathbb{E}[ \lVert Q_{k}^{A}-Q_{k}^{A_{L}}+Q_{k}^{A_{L}}-Q_{k}^{*}\rVert_{\infty}] \\
&\leq \mathbb{E}[ \lVert Q_{k}^{A_{L}}-Q^{*} \rVert_{\infty}]+\mathbb{E}[ \lVert Q_{k}^{A}-Q_{k}^{A_{L}} \rVert_{\infty}]\\
&\leq \mathbb{E}[ \lVert Q_{k}^{A_{L}}-Q^{*} \rVert_{\infty}]+\mathbb{E}[ \lVert Q_{k}^{A_{U}}-Q_{k}^{A_{L}} \rVert_{\infty}] 
\end{split}
\]
The second inequality is due to $Q_{k}^{A_{U}}-Q_{k}^{A_{L}} \geq 
Q_{k}^{A}-Q_{k}^{A_{L}}.$ This can be inferred from the fact that the lower comparison system and upper comparison system sandwich the original system as $Q_{k}^{L}-Q^{*}\leq Q_{k}-Q^{*}\leq Q_{k}^{U}-Q^{*}$.
Then we can rewrite the equation as
\[
\begin{split}
\mathbb{E}[ \lVert Q_{k}^{A}-Q^{*} \rVert_{\infty}] &\leq \mathbb{E}[ \lVert Q_{k}^{A_{L}}-Q^{*} \rVert_{\infty}]+\mathbb{E}[ \lVert Q_{k}^{A_{U}}-Q_{k}^{A_{L}} \rVert_{\infty}] 
\end{split}
\]
Combining this inequality with  ($\ref{eqn:lcs_upper_bound}$),  ($\ref{eqn:error_upper_bound}$) yields the following result:

\begin{align}\label{eqn:before_simplify}
\mathbb{E}[ \lVert Q_{k}^{A}-Q^{*} \rVert_{\infty}] &\leq \frac{16\gamma d_{\text{max}}\lvert \mathcal{S}\times \mathcal{A}\rvert\alpha^{1/2}}{d_{\text{min}}^{7/2}(1-\gamma)^{9/2}}+\frac{24\rho^{k-3}k^{3}\lvert \mathcal{S}\times \mathcal{A}\rvert^{3/2}}{(1-\gamma)}+\frac{4\alpha^{1/2}\lvert \mathcal{S}\times \mathcal{A}\rvert}{d_{\text{min}}^{1/2}(1-\gamma)^{3/2}}\nonumber\\&\quad+\frac{40\gamma d_{\text{max}}\lvert \mathcal{S}\times \mathcal{A}\rvert \alpha^{1/2}}{d_{\text{min}}^{9/2}(1-\gamma)^{11/2}}+\frac{20\rho^{k-4}k^{4}\alpha\gamma d_{\text{max}}\lvert \mathcal{S}\times \mathcal{A}\rvert^{3/2}}{1-\gamma} 
\end{align}
We can group some terms of (\ref{eqn:before_simplify}) as follows

\begin{align*}
&\frac{16\gamma\,d_{\max}\lvert \mathcal{S}\times \mathcal{A}\rvert\,\alpha^{1/2}}
      {d_{\min}^{7/2}(1-\gamma)^{9/2}}
 +\frac{40\gamma\,d_{\max}\lvert \mathcal{S}\times \mathcal{A}\rvert\,\alpha^{1/2}}
       {d_{\min}^{9/2}(1-\gamma)^{11/2}}
+\;\frac{4\,\alpha^{1/2}\lvert \mathcal{S}\times \mathcal{A}\rvert}
         {d_{\min}^{1/2}(1-\gamma)^{3/2}}
\;\le\;
3\biggl(\frac{40\,\lvert \mathcal{S}\times \mathcal{A}\rvert\,\alpha^{1/2}}
            {d_{\min}^{9/2}(1-\gamma)^{11/2}}\biggr).
\end{align*}
Other remaining terms can be grouped as follows
\begin{align*}
\frac{24\rho^{k-3}k^{3}\lvert \mathcal{S}\times \mathcal{A}\rvert^{3/2}}{(1-\gamma)}+\frac{20\rho^{k-4}k^{4}\alpha\gamma d_{\text{max}}\lvert \mathcal{S}\times \mathcal{A}\rvert^{3/2}}{1-\gamma}\leq 2\biggl(\frac{24\rho^{k-4}k^{4}\lvert \mathcal{S}\times \mathcal{A}\rvert^{3/2}}{(1-\gamma)}\biggl).
\end{align*}
Finally, we can get the finite-time error bound of SDQ
\begin{align*}
\mathbb{E}[ \lVert Q_{k}^{A}-Q^{*} \rVert_{\infty}] &\leq \frac{120\alpha^{1/2}\lvert \mathcal{S}\times \mathcal{A}\rvert}{d_{\text{min}}^{9/2}(1-\gamma)^{11/2}}+\frac{48\rho^{k-4}k^{4}\lvert \mathcal{S}\times \mathcal{A}\rvert^{3/2}}{(1-\gamma)}.
\end{align*}

\subsection{Proof of \cref{cor:final-theorem} (Finite-time error bound of SDQ)}\label{sec:final-corollaly}
We focus on the term
\[
\begin{split}
k^{4}\rho^{k-4}=\rho^{-4}\rho^{k/2}k^{4}\rho^{k/2}
\end{split}
\]
Let $f(x)=x^{4}\rho^{x/2}$. Solving the first-order optimality condition
\[
\frac{df(x)}{dx}=\frac{d}{dx}x^{4}\rho^{x/2}=4x^{3}\rho^{x/2}+x^{4}\frac{1}{2}\rho^{x/2}\ln{(\rho)}=0
\]
we have that its stationary points are $x=\frac{-8}{\ln{(\rho)}}$ and $x=0$. The corresponding function values are 
\[
f\biggl(\frac{-8}{\ln{(\rho)}}\biggl)=\frac{(-8)^{4}}{(\ln{(\rho)})^{4}}\rho^{\frac{-4}{\ln{(\rho)}}},\quad f(0)=0
\]
Moreover, solving the second-order optimality condition
\begin{align*}
\frac{d^{2}f(x)}{dx^{2}}&=\frac{d}{dx}\biggl(4x^{3}\rho^{x/2}+x^{4}\frac{1}{2}\rho^{x/2}\ln{(\rho)}\biggl)=12x^{2}\rho^{x/2}+4x^{3}\rho^{x/2}\ln{\rho}+\frac{1}{4}x^{4}\rho^{x/2}((\ln(\rho)))^{2},
\end{align*}
we have $f''(\frac{-8}{\ln{(\rho)}})<0$ and $f''(0)=0$. Therefore, one concludes that $f(\frac{-8}{\ln{(\rho)}})$ is the unique local maximum point. Because the function is continuous and converges to zero as $x\rightarrow+\infty$, it is bounded. This implies that $x=\frac{-8}{\ln{(\rho)}}$ is a global maximum point. Then, we have
\[
\rho^{(k-4)}k^{4}=\rho^{-4}\rho^{k/2}k^{4}\rho^{k/2}\leq \rho^{-4}{\frac{(-8)^{4}}{(\ln{(\rho)})^{4}}\rho^{\frac{-4}{\ln{(\rho)}}}}\rho^{k/2}.
\]
Combining this bound with (\ref{eqn:final-theorem}), one get the bound in (\ref{eqn:final-theorem-cor}).

\bibliographystyle{elsarticle-num} 

\end{document}